\documentclass{statsoc}
\usepackage[a4paper]{geometry}
\RequirePackage{xcolor}
\RequirePackage{listings}
\RequirePackage{mathtools}
\RequirePackage[11pt]{moresize}
\RequirePackage{forest}
\RequirePackage{graphicx}
\RequirePackage{algpseudocode,algorithm,algorithmicx}
\RequirePackage{subcaption}
\RequirePackage{array}
\RequirePackage{dsfont}
\RequirePackage{gauss}
\usepackage[title]{appendix}
\usepackage{natbib}
\RequirePackage[colorlinks,citecolor=blue,urlcolor=blue]{hyperref}

\newtheorem{theorem}{Theorem}
\newtheorem{prop}{Proposition}

\newtheorem{definition}{Definition}

\newtheorem{rem}{Remark}

\algrenewcommand\algorithmicrequire{\textbf{Precondition:}}
\algrenewcommand\algorithmicensure{\textbf{Postcondition:}}

\title[Automatic deforestation detectors on frequentist statistics]{Automatic deforestation detectors based on frequentist statistics and their extensions for other spatial objects.}

\author[Jesper Muren {\it et al.}]{Jesper Muren, https://orcid.org/0000-0002-9208-5325}
\address{Department of Mathematics, Stockholm University,
Stockholm,
Sweden.}
\author[Vilhelm Niklasson]{Vilhelm Niklasson, https://orcid.org/0000-0001-9228-0369}
\address{Department of Mathematics, Stockholm University,
Stockholm,
Sweden.}
\author[Dmitry Otryakhin]{Dmitry Otryakhin \thanks{Address for correspondence: Dmitry Otryakhin, Department of Mathematics, Stockholm University, Matematiska institutionen, 106 91 Stockholm, Sweden. E-mail: d.otryakhin.acad@protonmail.ch}, https://orcid.org/0000-0002-4700-7221}
\address{Department of Mathematics, Stockholm University,
Stockholm,
Sweden.}
\email{d.otryakhin.acad@protonmail.ch}
\author[Maxim Romashin]{Maxim Romashin, https://orcid.org/0000-0003-2976-3346}

\begin{document}
\maketitle
\begin{abstract}
This paper is devoted to the problem of detection of forest and non-forest areas on Earth images. We propose two statistical methods to tackle this problem: one based on multiple hypothesis testing with parametric distribution families, another one---on non-parametric tests. The parametric approach is novel in the literature and relevant to a larger class of problems---detection of natural objects, as well as anomaly detection. We develop mathematical background for each of the two methods, build self-sufficient detection algorithms using them and discuss practical aspects of their implementation. We also compare our algorithms with those from standard machine learning using satellite data. 
\end{abstract}


\section{Introduction}

Within the last decade, there have been some attempts to create remote deforestation detectors. Contemporary literature on the topic seems to be geared towards machine learning approaches. Just to name a few examples we refer to \cite{Ortega20} that employed deep learning for deforestation detection in the Brazilian Amazon and \cite{Shermeyer15} that presented k-nearest neighbour classification to track deforestation in Peru. Bayesian approach has also seen use, e.g. in \cite{Reiche15} on time series data from Fiji. The present work aims to create methods having statistical foundation, and, unlike machine learning methods, are fast, simple and in principle interpretable while being comparably accurate. The qualities are important on their own, but may also be useful for the task of near-real-time deforestation detection. To achieve these characteristics we employ sample means and covariances and empirical characteristic functions (ECFs) of colour intensities of image samples. The idea behind it is that forests seen from above, like many other natural objects, show rather simple homogeneous structure. Moreover, in many cases natural objects are distinguishable due to their colour features rather than their spatial patterns. Thus, operating on distributions and their ECFs and completely erasing all the spatial dependencies seems to be reasonable.

The general statistical setting of the problem is as follows.
In visible spectrum images of geographical objects are represented as three matrices. Each of them consists of intensities of one of the red, green or blue colours. Thus, every pixel of an Earth image is represented by three numbers: one in each of the matrices. We employ supervised learning where tags (forest / non-forest) are assigned to entries of the training data and learn features of the distributions of the matrices in those entries. During classification, test images are split into much smaller square sub-images for the purpose of obtaining better resolution, and these sub-images are tested against reference images that are known to contain forest.

We intend to investigate two different approaches to the problem: non-parametric multiple hypothesis testing and parametric one. In the case of parametric multiple hypothesis testing, we assume parametric distributions for the colour intensities of pixels and use training data to estimate the parameters of said distributions. New data is then tested against this distribution. In the case of the non-parametric setting, we do not assume any known distribution for the colour intensities but instead test if the multivariate distributions of the colour intensities of the training data and new data have equal means and covariance matrices. 

We build both the parametric and the non-parametric algorithms in such a way that they work on images in visible spectrum, in particular, RGB composites. Thus, the algorithms do not require any certain pre-processing technique for the data and can be applied to a broad range of sources of images of Earth.
There are at least two ways of obtaining large datasets of Earth images: satellite data obtained through open access Earth observation missions, e.g. Landsat (\cite{Landsat8}) and Sentinel-2 (\cite{Sentinel2}) and aerial photo imaging.
In this work, Sentinel-2 images are used because they provide good spatial and temporal resolution. Sentinel-2 mission provides data from visible, near infrared and short-wave infrared sensors in 13 corresponding spectral bands. Although we use RGB composites obtained via a post-processing procedure of Sentinel-2 original tiles, our methods could work on all of the data in those 13 bands. 

\subsection{A practical example}
\label{subsc:PractExampl}
In order to give more intuition of what task we tackle, we present a real-life example of deforestation detection. Figure \ref{fig:Defore_example_a} shows a Sentinel-2 image of a piece of land near the city of Arkhangelsk in Russia suffering from logging. Figure \ref{fig:Defore_example_b} is a black-and-white image representing the detection result. White represents forest and black is used for all types of terrains not covered by it.
\begin{figure}[H]
	\subfloat[Forest image analysed]{\includegraphics[width= 0.48\textwidth]{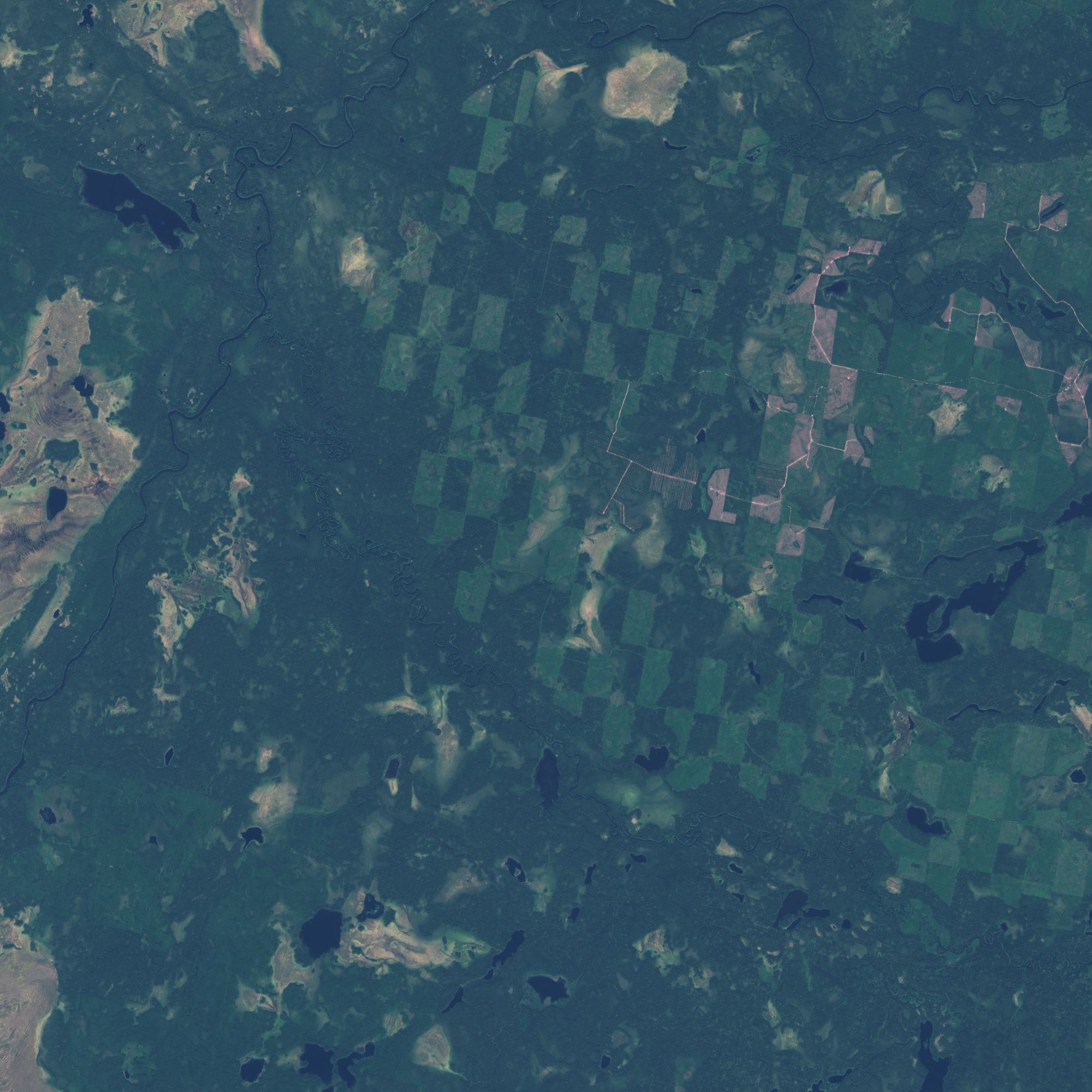} \label{fig:Defore_example_a}}
	\subfloat[The result of detection. White---forest, black---non-forest]{\includegraphics[width= 0.48\textwidth]{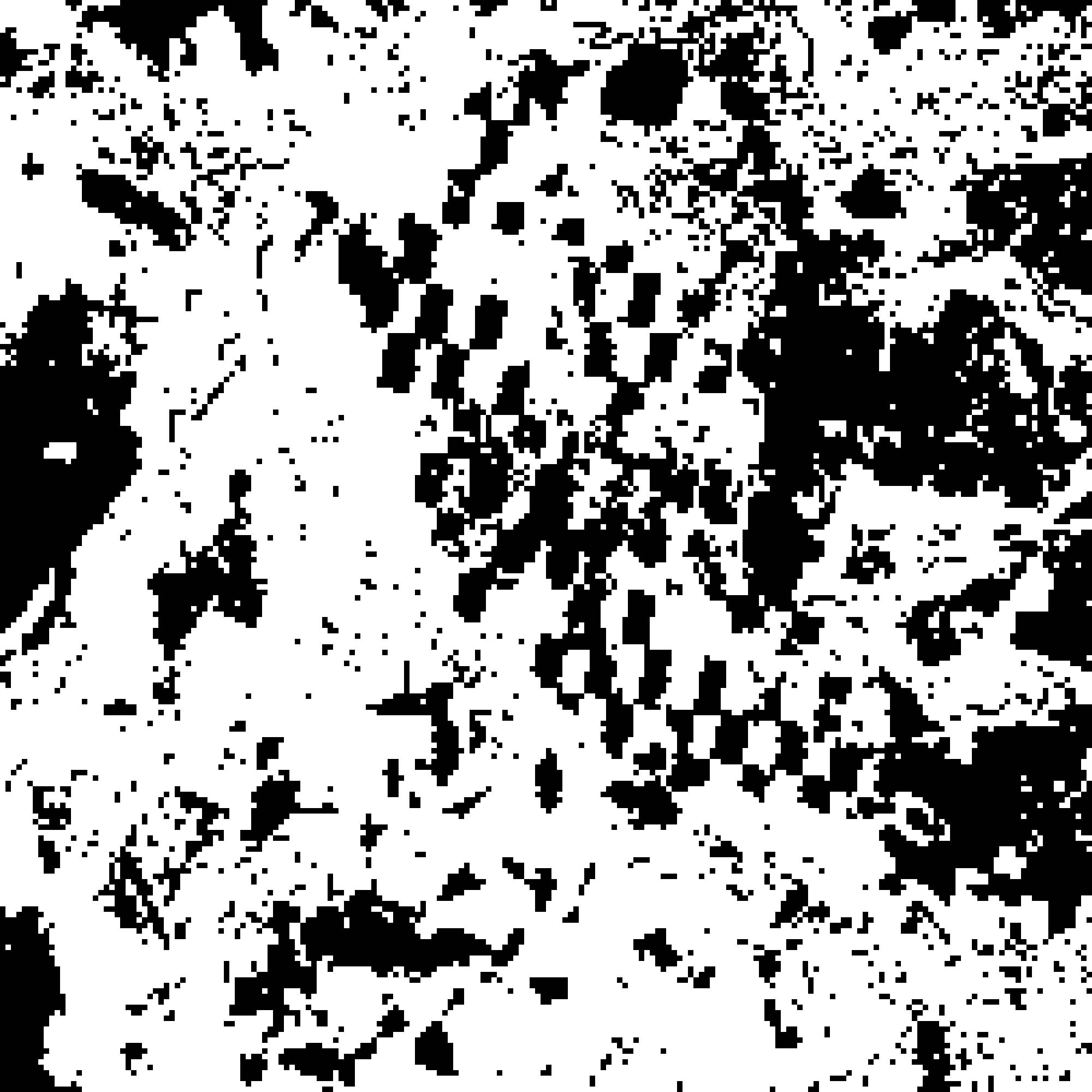} \label{fig:Defore_example_b}}\\
	
	\caption{A plot of land with deforestation and the result of detection produced by the non-parametric algorithm.}
	\label{fig:Defore_example}
\end{figure}

One can see that the detector highlighted the areas with logging traces (black rectangles) as well as natural objects: lakes, bare ground and pieces of rivers.

The rest of the paper is organized as follows. Section \ref{EDA} briefly mentions the key steps in data preparation and discusses what features forest data possess. Sections \ref{Sec:nonPar} and \ref{Sec:param est} describe correspondingly novel non-parametric and parametric methods. Section \ref{Sec:TheorCompar} discusses theoretical properties of the two methods, while the last section compares the two with well-known methods from machine learning using Satellite data.

Our main contribution contains the application of the known theory of Hotelling's $T^2$ statistic to the problem of deforestation detection, which occupies Section \ref{Sec:nonPar}. There, we prove Theorem \ref{thm:NonPar_test} which gives the asymptotic distribution of a distance between two properly scaled multivariate populations. Using Theorem \ref{thm:NonPar_test} we build classification and training algorithms in \ref{subsc:nonpar_hypoth_test} and \ref{subsc:nonpar_train_cv}. Another major contribution is Section \ref{Sec:param est} (excluding \ref{subsc:StableReview}), where we develop a new setting of hypothesis testing for stable models based on empirical characteristic functions in \ref{subsc:Param_MahalaDist} and use this setting to build algorithms for detection of forest in \ref{subsc:ParamHypTest} and \ref{subsc:par_train_cv}.

\section{Data and exploratory analysis}
\label{EDA}
To obtain data in visible spectrum we use a simple preprocessing procedure, collecting red, green and blue bands from Sentinel-2 Level-1C products and combining them into one GeoTIFF file. It is then handled by GDAL (\cite{gdal}), which creates three numeric matrices (one per each of the red, green and blue channels) with values between 0 and 1. The procedure is described in Appendix \ref{sec:Data_pre-processing} in more detail.
\begin{figure}[h!]
	\subfloat[Forest image analysed]{\includegraphics[width=0.42\textwidth]{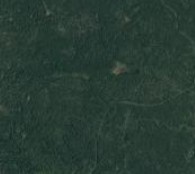}}
	\subfloat[Comparison of densities for red colour intensity]{\includegraphics[width=0.44\textwidth]{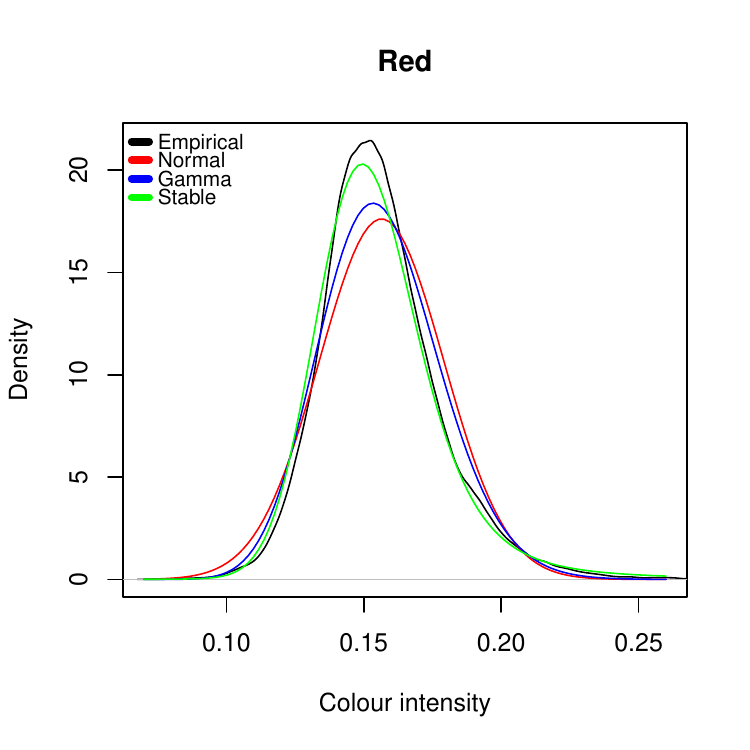}}\\
	\subfloat[Comparison of densities for green colour intensity ]{\includegraphics[width=0.44\textwidth]{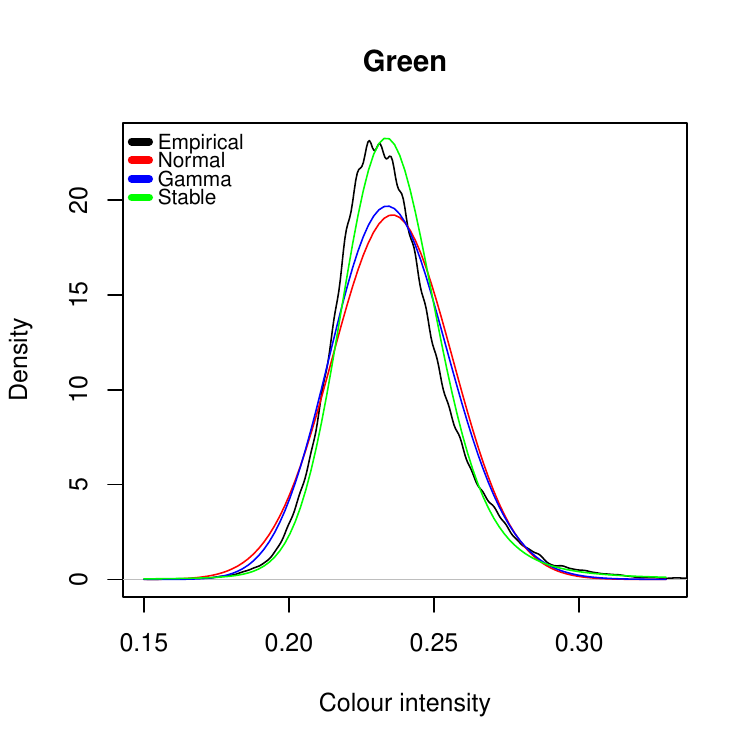}}
	\subfloat[Comparison of densities for blue colour intensity]{\includegraphics[width=0.44\textwidth]{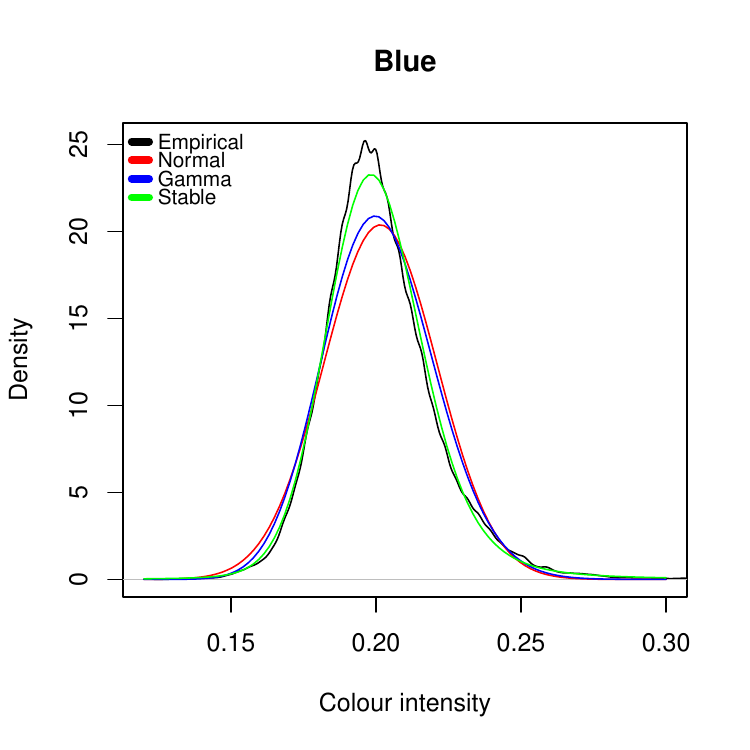}}
	
	\caption{Plots of estimated probability density functions for distribution of colour intensities in an image of a forest.  The means of the red, green and blue intensities are respectively 0.16, 0.24 and 0.20.}
	\label{fig:EDA}
\end{figure}

Initial inspection of the empirical density of the forest data gave us some additional ideas for potentially suitable distributions. We considered distributions including, but not limited to, stable, normal, log-normal, gamma, Cauchy, Weibull and Levy. We estimated parameters for each of the three colour intensity distributions separately using Koutrouvelis regression-type estimator for the stable distributions, discussed in Section \ref{param est}, and by maximum likelihood for the others. Then we calculated the root-mean-square-errors (RMSE) 
$$
\text{RMSE}=\sqrt{\frac{\sum_{i=1}^m(\hat{f}(x_i)-f^*(x_i))^2}{m}}
$$
where $\hat{f}(\cdot)$ is the probability density function using the estimated parameters, $f^*(\cdot)$ is the empirical density and $x_i$ is colour intensity of data point $i=1,\dots,m.$ The empirical density is estimated using kernel density smoothing. To choose the best fitting distributions we then plotted the densities of the three distributions with the lowest RMSE's together with the empirical densities and compared. These distributions appeared to be normal, gamma and stable consistently throughout the set of images. An example of a forest image with such comparison can be seen in Figure \ref{fig:EDA}, where we can observe the three best fitting distributions for a forest image with all parameters estimated. Further from the plots we see that the normal distributions do not seem to fit very well while the gamma and stable distributions are quite close, with stable edging out gamma especially for the intensity of the colour red.

Since the focus in this work is to be able to detect if an area has forest or not, we  describe detection of presence or lack of this type of natural objects. However, our exploratory data analysis shows us that the framework developed could be used to detect other natural objects from satellite images as well. In Figures \ref{fig:Mountain} and \ref{fig:Sea} in Appendix \ref{sec:EDA_adds}, we can see promising results of using the stable distribution to describe the colour intensities of images of mountains and sea. When it comes to man-made objects such as, for example a city, our approach does not work, and we have to at least assume a multimodal mixture of some kind to be appropriate since satellite images contain distinctly different parts, i.e., an orange roof, green grass and grey pavement.

\section{Detection with a non-parametric method}
\label{Sec:nonPar}
In this section, the non-parametric method for forest detection is developed. Later, in Section \ref{Sec:param est}, the ideas presented here are used to create the theory and implementation for the parametric one.
\subsection{Comparing images with the scaled, squared Mahalanobis distance}
\label{Hotelling}
The Mahalanobis distance $D$ was originally introduced as a distance between a given multivariate point $\boldsymbol{q}$ and a distribution. For a point $\boldsymbol{q}$ and a distribution with mean vector $\boldsymbol{\mu}$ and covariance matrix $\boldsymbol{\Sigma};$ the distance is given by 

$$
D=\sqrt{(\boldsymbol{q}-\boldsymbol{\mu})^T\boldsymbol{\Sigma}^{-1}(\boldsymbol{q}-\boldsymbol{\mu})}.
$$
Where, in practice, the mean and covariance matrix of the distribution would be estimated by the sample mean and covariance matrix. Further extensions can be made to calculate the distance between two multivariate populations with

\begin{equation}\label{mahala dist}
D=\sqrt{(\boldsymbol{\bar{X}}_1-\boldsymbol{\bar{X}}_2)^T\boldsymbol{\hat{\Sigma}}^{-1}(\boldsymbol{\bar{X}}_1-\boldsymbol{\bar{X}}_2}).
\end{equation}
Where $\boldsymbol{\bar{X}}_1$ and $\boldsymbol{\bar{X}}_2$ are the sample mean vectors of the two populations. The variable $\boldsymbol{\hat{\Sigma}}^{-1}$ is the inverse of the pooled sample covariance matrix, given by $\boldsymbol{\hat{\Sigma}}=(n_1\boldsymbol{\hat{\Sigma}_1}+n_2\boldsymbol{\hat{\Sigma}_2})/(n-2)$ for samples with respective sample sizes $n_1, n_2$, total sample size $n=n_1+n_2$ and sample covariance matrices $\boldsymbol{\hat{\Sigma}_1}$, $\boldsymbol{\hat{\Sigma}_2}$ respectively. An interested reader can find more on the topic in \cite{Mardia79}.

The Mahalanobis distance is often used in cases where there is correlation present in the considered multivariate distribution. As the reader might have noticed, the Mahalanobis distance expression is closely related to the statistical concept of standardization, i.e, subtracting the mean and dividing by the standard deviation. The idea is to decorrelate and scale each dimensions variance to be 1 and then calculate the Euclidean distance.

\begin{theorem} \label{thm:NonPar_test}
	Let $\boldsymbol{X_1}$ and $\boldsymbol{X_2}$ be independent samples of sizes $n_1 \times p$ and $n_2 \times p$ 
	composed from i.i.d. p-dimensional random vectors from distributions with finite mean $\boldsymbol{\mu}$ and finite covariance matrix $\boldsymbol{\Sigma}$. Then, for $D$ from (\ref{mahala dist}) we have that $\frac{n_1n_2}{n_1+n_2}D^2 \xrightarrow{d} \chi^2(p)$ as $n_1 \to \infty$ and $n_2 \to \infty$ at the same rate.
\end{theorem}
For a proof, we refer to Appendix \ref{sec:Proofs}.
\begin{rem}
	Note, that $\boldsymbol{X_1}$ and $\boldsymbol{X_2}$ do not have to be drawn from the same distribution. If they are drawn from different ones, but possessing the same mean and covariance matrix, the result of the theorem still holds.
\end{rem}

\subsection{Hypothesis testing based on Mahalanobis distances between images}
\label{subsc:nonpar_hypoth_test}
Suppose we have an RGB image of forest (Figure \ref{fig:non-param_flowchart}) of the size $n \times p$ pixels. It consists of three matrices corresponding to each of the red, green and blue channels. The matrices contain values of colour intensities of pixels in the range $[0,1]$. Then, each matrix is split into column-vectors that are subsequently stack together to form one column-vector of the length $np$. The three obtained large vectors are bound to form a $np \times 3 $ matrix $\boldsymbol{X}$. Since the same operation is applied to every matrix, the triplets of pixel intensities are kept unchanged, so it is possible to analyse dependencies between colours. On the contrary, matrix $\boldsymbol{X}$ does not contain any information about spatial locations of the pixels. From $\boldsymbol{X}$, a three-dimensional mean vector and a $3 \times 3$ covariance matrix are calculated.
\begin{figure}[h!]
	\centering
	\includegraphics[width=0.7\textwidth]{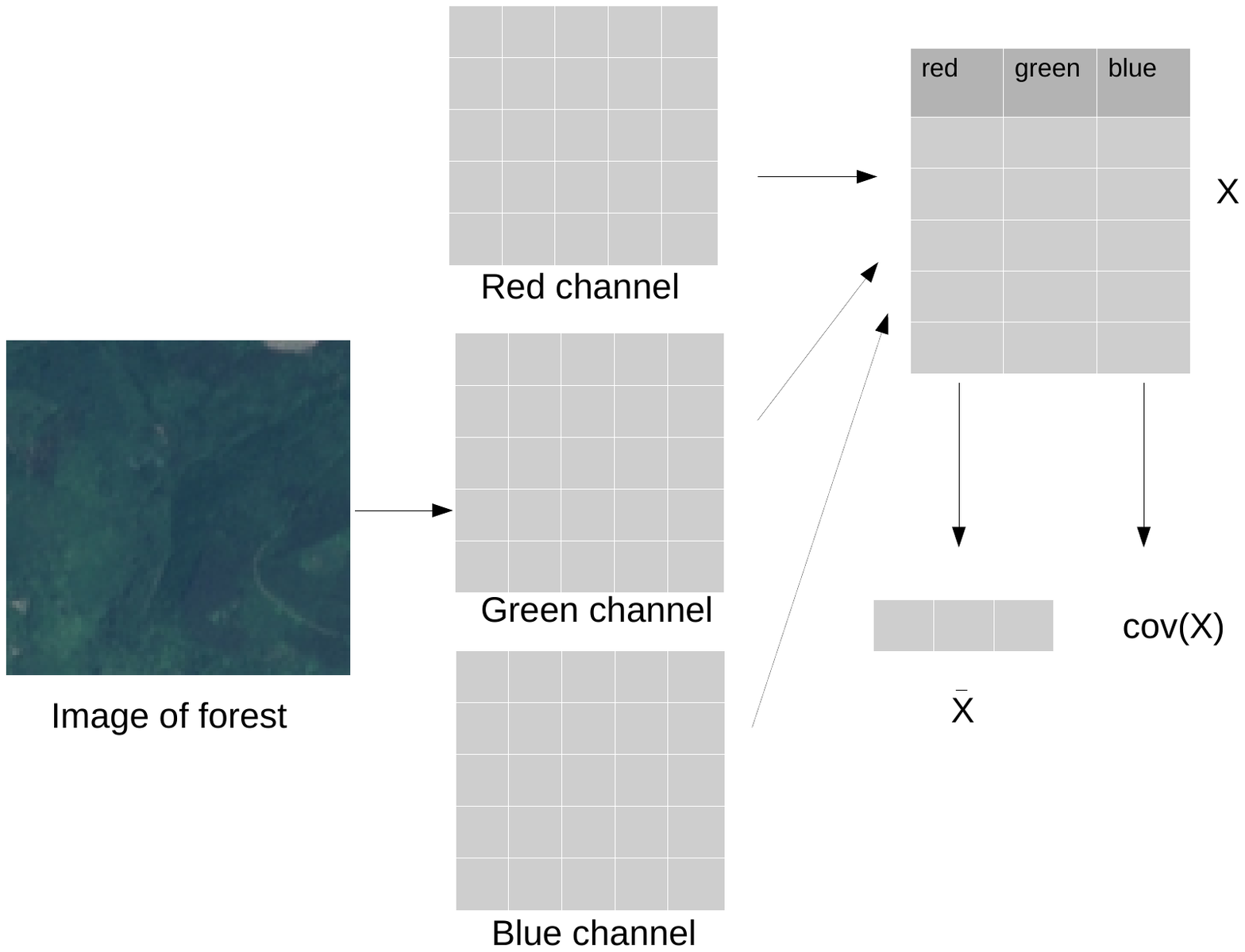}
	\caption{Data processing for the non-parametric testing}
	\label{fig:non-param_flowchart}
\end{figure}

\begin{figure}[ht!]
	\centering
	\includegraphics[width=.8\textwidth]{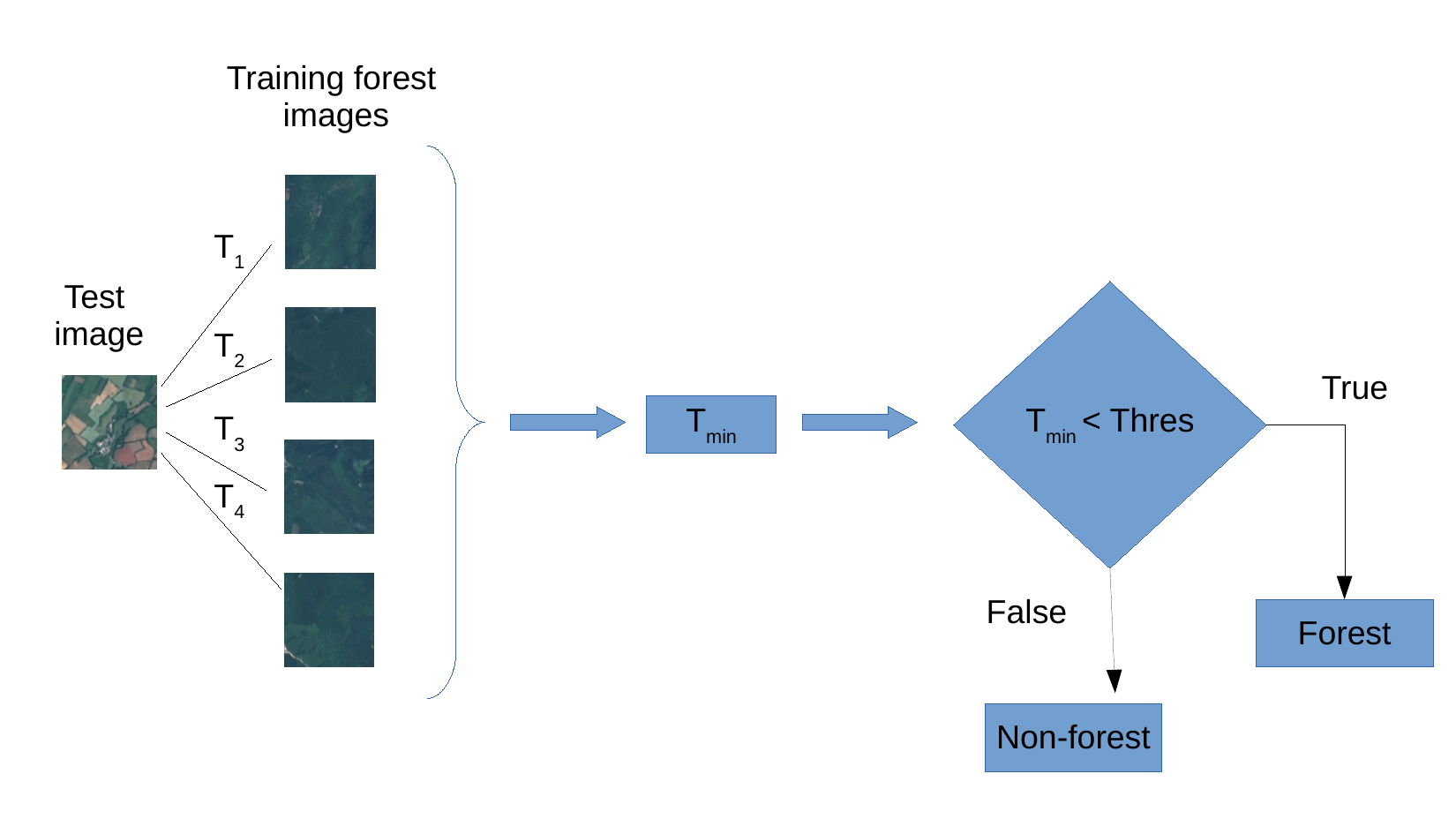}
	\caption{Classification in the non-parametric model}
	\label{fig:mult_hyp_test}
\end{figure}
The assumptions of Theorem \ref{thm:NonPar_test} hold for data of type of $\boldsymbol{X}$: since all entries are bounded between $0$ and $1$, the means and covariances will always be finite. This naturally leads us to a hypothesis testing framework. Having two images, one picture that is known to depict forest and one test image, we can use the theorem to accept or reject the null hypothesis that the test image is forest. This is done by first converting RGB composites into matrices $\boldsymbol{X_1}$ and $\boldsymbol{X_2}$, obtaining their sizes, means and covariance matrices, and second, computing the Mahalanobis distance between $\boldsymbol{X_1}$ and $\boldsymbol{X_2}$, and checking whether or not the rescaled Mahalanobis distance  $T:=\frac{n_1n_2}{n_1+n_2}D^2$ is lower than a certain threshold. 

A scheme for multiple hypothesis testing framework is depicted in Figure \ref{fig:mult_hyp_test}. When it comes to classification of a certain image, rescaled Mahalanobis distances $T_1, T_2, \dots$ between the image and every forest picture in the training set are computed. Out of all $T$'s, the minimal one is chosen, and is then compared to a threshold obtained during model training. If the minimal $T$ is smaller that the threshold, the test image is tagged forest, if greater, then---non-forest.
\begin{rem}
	This setting is equivalent to multiple hypothesis testing with the thresholds set to be equal. This restriction can be dropped, in which case the precision of predictions will generally increase. But the downside of such a generalization would be that the additional computational complexity during both training and classification would increase linearly with the number of forest images.
\end{rem}
\begin{rem}
	This algorithm describes dependences between colours by means of the pooled sample covariance matrix. On the contrary, all the information about spatial features are erased. Also, distributions of colour intensities in each channel are described only by their location parameters.
\end{rem}

\subsection{Training and cross-validation of the non-parametric model}
\label{subsc:nonpar_train_cv}
Training of the non-parametric model is schematically depicted in Figure \ref{fig:mult_hyp_train}. 
\begin{figure}[h!]
	\centering
	\includegraphics[width=0.8\textwidth]{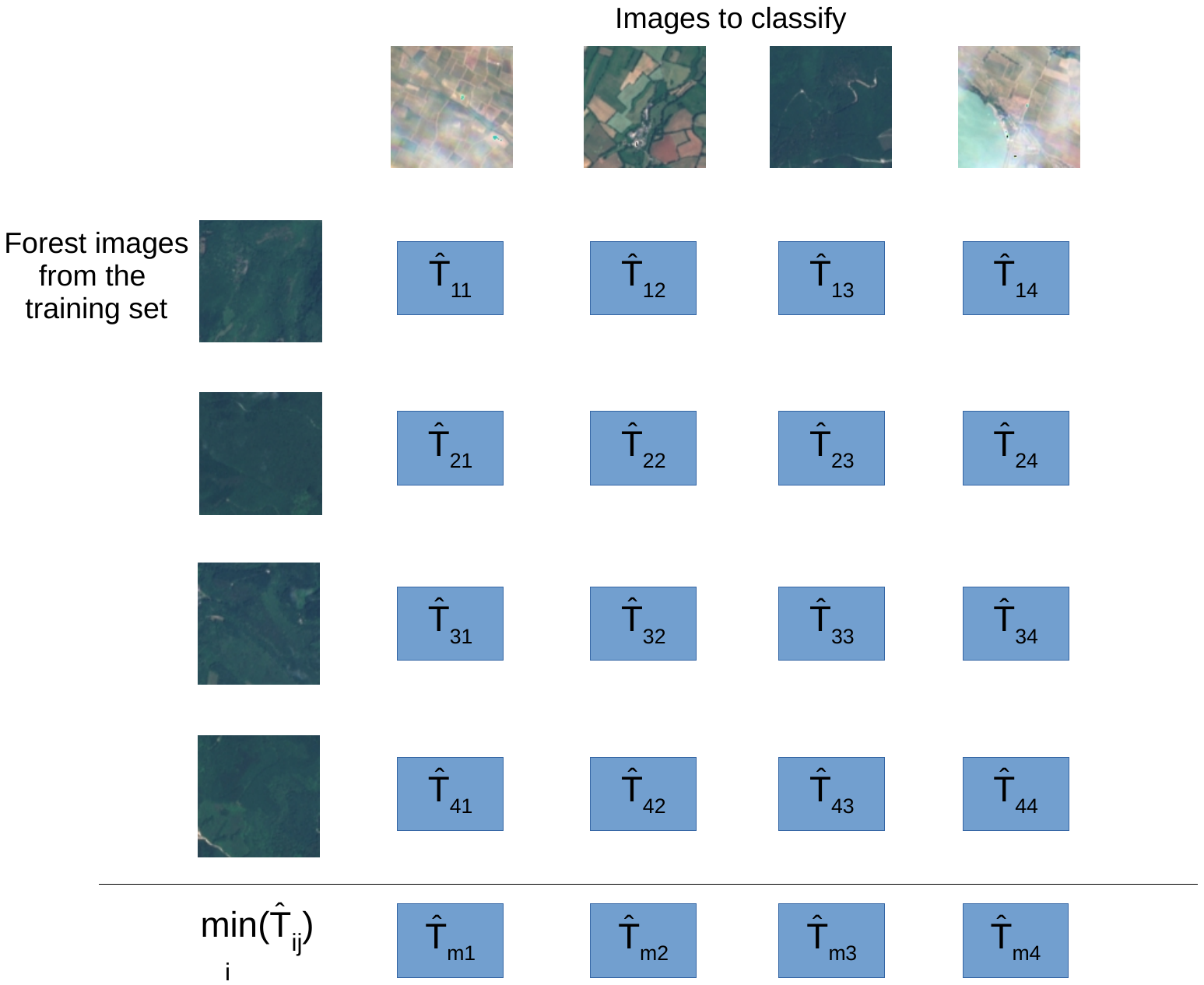}
	\caption{Training in the non-parametric model}
	\label{fig:mult_hyp_train}
\end{figure}
One can observe two sets of images there: reference forest images in the left column (set 1) and images of both types in the upper row (set 2). For every pair i-th image from set 1 and j-th image from set 2, statistic $\hat T_{ij} = \frac{n_1n_2}{n_1+n_2}D_{ij}^2$ is computed, where $D_{ij}$ is the Mahalanobis distance between the two pictures. Then, for every image from set 2, the minimal statistic $\min_i \hat T_{ij}$ is found. Lastly, the threshold $T$ is chosen to minimize misclassification error when classifying images from set 2 according to a result of evaluation of Boolean expression $\min_i \hat T_{ij}<T$. As previously, True corresponds to forest, while False---to non-forest.

Aside from simple training on a set of images, we adopt the setting of k-fold cross-validation for this model. Initially, we shuffle pictures and group them into $k$ folds. In this procedure, in order to preserve internal features, we don't split or merge pictures. We follow a similar procedure to what was described above. Once a fold is chosen, set 1 consists of all forest pictures in the other $k-1$ folds, while set 2 is the chosen holdout fold. Then statistics $\min_i \hat T_{ij}$ are computed for each of the $k$ holdout sets and $T$ is chosen to minimize the misclassification error on all the combinations. Cross-validation is summarized in Algorithm \ref{alg:np}.
\begin{algorithm}[h!]
	\caption{Non parametric cross validation procedure
		\label{alg:np}}
	\begin{algorithmic}[1]
		\State \textbf{Inputs:} Data set of $N$ images labelled forest or non-forest, number of folds $k$ \& max threshold range $T_{max}$.
		\State Compute means and covariances for every image.
		\State Randomly shuffle images and group them into $k$ folds.
		\For{$i =1$ to $k$}
		\For{each image in fold $i$}    
		\State \parbox[t]{313pt}{Compute $\hat T = \frac{n_1n_2}{n_1+n_2}D^2$ from Theorem \ref{thm:NonPar_test}, between the image and all images labelled forest in the remaining $k-1$ folds.\strut}
		\State Keep only the smallest $\hat T$ for the image.
		\EndFor
		\EndFor
		\For{$t$ in $0$ to $T_{max}$}
		\State \parbox[t]{313pt}{For all images and all $\hat T$'s, tag images with $\hat T<t$ as forest and $\hat T>t$ as non forest. \strut}
		\State Compute the accuracy for the given threshold $t,$ $\text{Acc}(t).$ 
		\EndFor
		\State \Return $\text{arg}\max_{t}\text{Acc}(t)$ and $n$'s, means and covariances of forest images.
	\end{algorithmic}
\end{algorithm}

\section{Detection with a parametric method}
\label{Sec:param est}
The non-parametric approach described in Section \ref{Sec:nonPar} is simple and universal in terms of underlying data, it doesn't make any assumptions on the data coming from RGB composites. On the other hand, it possesses a theoretical drawback: it doesn't take advantage of the distribution of the data, only its mean and covariance matrix. The parametric approach to the detection was designed basing on the opposite ideas: to capture the actual distributions of the data and parametrize them in order to make use of parametric theoretical machinery.

\subsection{Stable distributions for signal processing}
\label{subsc:StableReview}
The second classification method we develop in this study is based on univariate stable distributions as they proved to fit well to the empirical data (Section \ref{EDA}). In this short section we give some context for use of stable distributions in the field of signal processing, describe how they are fit to satellite data in the present paper and how we estimate sampled data.

The stable distribution family is a class generalising the Gaussian one. In signal processing it appears due to two factors. First, stable distribution is a result of the Generalised Central Limit Theorem which says that if there exists any limit of a mean of i.i.d. variables with possibly unbounded variations, then this limit must have a stable distribution. Thus, this class is a natural choice for modelling signals with potentially large deviations. Second, stable distributions might be asymmetric, and it is controlled by one of the parameters. This family of random variables has already seen use in remote sensing areas such as ship detection (\cite{Wang08}) and video background subtraction (\cite{bhaskar2010}). While it is  unclear to us what role the Generalised Central Limit Theorem plays in shaping the distributions of forest colour intensities, it is certain that parametric control of their shape provides accurate fitting (Figure \ref{fig:EDA}) when the general stable distribution is assumed.

\begin{definition}
	\label{def:stab_rv}
	A random variable $X$ is said to have a stable distribution if there are parameters $0 < \alpha \leq 2,$ $\sigma \geq 0$, $-1\leq \beta \leq 1$ and $\delta \in \mathds{R}$ such that its characteristic function has the following form: 
	\begin{equation} 
	\label{stable def2}
	\normalfont
	\phi(t)= E[e^{itX}]=
	\begin{cases}
	\text{exp}\bigg\{-\sigma ^\alpha|t|^\alpha(1-i\beta\text{sgn} (t) \text{tan}(\frac{\pi\alpha}{2}))+i\delta t\bigg\} & \text{if $\alpha \neq 1$}\\
	\text{exp}\bigg\{-\sigma|t|(1+i\beta \frac{2}{\pi}\text{sgn} (t)\text{log}(|t|))+i\delta t\bigg\} & \text{if $\alpha = 1$.}
	\end{cases}
	\end{equation}
\end{definition}
Where sgn$(t) = 0$ if $t=0$, $-1$ if $t<0$ and $1$ if $t>0.$ The parameter $\alpha$ is known as the index of stability and controls the decay of the tails. The parameters $\beta$ is not the classical statistical skewness but it indicates how the distribution is skewed, indicating left skewness if $\beta < 0$, right skewness if $\beta > 0$ and symmetry if $\beta = 0.$ The parameter $\sigma$ is a scale parameter controlling variability and $\delta$ a location parameter which shifts the distribution. The location parameter only corresponds to the mean when $\alpha > 1$, otherwise the mean is undefined. In fact, for $\alpha < 2$ the $p$-th absolute moment $E[|X|^p]<\infty$  iff $p<\alpha.$

Apart from a few cases, the closed form probability function is unknown for stable laws. There exists a method for computing stable densities numerically (\cite{Nolan}) but it is computationally demanding. That is why we developed a mathematical foundation of our parametric method based on CF's and ECF's.
The most famous special cases where the probability density function can be expressed explicitly is when the stable distribution follows a Gaussian, Cauchy or Lévy distribution. A stable distribution is Gaussian when $\alpha = 2$.

\subsubsection{Parameter estimation}
\label{param est}
This section will cover estimation of the four parameters of the univariate stable distribution used in this article. There is a number of techniques available for estimation of these parameters, including quantile-based estimators, maximum-likelihood-based estimators and those based on the method of moments. For this work, the Koutrouvelis regression-type estimator (\cite{Kout}) was chosen, as it shows a great balance between precision and computational performance. Estimators which performed well but were very slow were also studied, including ones based on the generalized method of moments. Similarly, estimators which were faster than the Koutrouvelis one were assessed, but these did not perform well enough for our purposes: the McCulloch estimator (\cite{mcculloch}) based on quantiles and the one developed by \citeauthor{S_J_Press_72}. A more in-depth study on the performance of different parameter estimators may be of interest to individuals interested in improving the algorithm proposed in this article.

The Koutrouvelis regression-type estimator is quite fast, but one big downside is that the estimator of the location parameter $\delta$ is often unreliable. For our purposes this turns out not to be a big issue though because experimentally we have seen that the parameter $\alpha > 1$ in all our forest data and when this is the case the location parameter $\delta$ equals the mean of the distribution \cite[Property 1.2.16]{Taqqu}. Which implies that the sample mean is a consistent estimator for the location parameter $\delta$ for our data.

\subsection{A test for empirical characteristic functions of stable distributions}
\label{subsc:Param_MahalaDist}
In this part, we develop a statistical test that rejects a null hypothesis that a sample comes from a stable distribution defined by (\ref{stable def2}) with certain parameters.

First, we provide a limit theorem for vectors composed of parts of ECFs of stable random variables. 
For a sequence $\{X_k\}_{k=1, \dots}$ of i.i.d. stable random variables  the ECF is given by 
\begin{equation}
\phi_n(t) := \frac{1}{n}\sum_{k=1}^{n}[\cos(tX_k) + i\sin(tX_k)].
\end{equation}
We define new random vectors
\begin{equation}
\boldsymbol{Z_n}(t) := \frac{1}{n}\sum_{k=1}^{n}
\begin{bmatrix}
\cos(tX_k) \\
\sin(tX_k)
\end{bmatrix}
\quad
\boldsymbol{Z_0}(t) :=
\begin{bmatrix}
\mathrm{Re}[\phi(t)] \\
\mathrm{Im}[\phi(t)]
\end{bmatrix}
\end{equation}
\begin{prop}
	\label{prop:z_n}
	Fix $t \in R$, then
	$\sqrt{n}\left[\boldsymbol{Z_n}(t)-\boldsymbol{Z_0}(t)\right]\xrightarrow{d} \mathcal{N}(0,\boldsymbol{\Sigma_z})$, where
	\begin{equation}
	\label{mtrx:Sigma_z}
	\boldsymbol{\Sigma_z}=
	\begin{bmatrix}
	\sigma_{11}  & \sigma_{12} \\
	\sigma_{21}  & \sigma_{22}
	\end{bmatrix}
	\end{equation}
	and the elements of the matrix are given by
	\begin{equation}
	\begin{split}
	\sigma_{11} =& \frac{1}{4} \left[ \phi(2t) + 2 + \phi(-2t) - (\phi(t))^2 -2\phi(t)\phi(-t) - (\phi(-t))^2 \right] \\
	\sigma_{22} =& \frac{1}{4} \left[ [\phi(t)]^2 - 2\phi(t)\phi(-t) +[\phi(-t)]^2 \right] - \frac{1}{4}[\phi(2t) + \phi(-2t) - 2] \\
	\sigma_{12} = \sigma_{21} =& \frac{1}{4i} \left[ \phi(2t) - [\phi(t)]^2 - \phi(-2t) + [\phi(-t)]^2 \right].
	\end{split}
	\end{equation}
\end{prop}
A proof can be found in Appendix \ref{sec:Proofs}.
\begin{rem}
	Note that when  $X_j, j=1, \dots, n$ are also symmetric and centred around the origin $\boldsymbol{\Sigma_z}$ reduces to the form 
	\begin{equation}
	\boldsymbol{\Sigma_z}=
	\begin{bmatrix}
	\frac{1}{2} \left[ 1 + \phi(2t) - 2(\phi(t))^2 \right]  & 0 \\
	0  & \frac{1}{2}[1- \phi(2t)]
	\end{bmatrix}
	\end{equation}
	This case was described in Section 3 of \cite{S_J_Press_72}.
\end{rem}

Proposition \ref{prop:z_n} provides a limiting distribution for a two-dimensional vector $\boldsymbol{Z_n} - \boldsymbol{Z_0}$, making the result inconvenient for statistical inference. To address that, we define an analogy of Mahalanobis distance, a statistic characterising how far $\boldsymbol{Z_n}$ is from $\boldsymbol{Z_0}$. The following theorem provides a limiting distribution for this statistic and plays the key role in the parametric inference.
\begin{theorem}
	\label{thm:param_test}
	Fix $t$. Assume that $\boldsymbol{\hat \Sigma_{z}} \xrightarrow{d} \boldsymbol{\Sigma_{z}}$ and that $\boldsymbol{Z^{(1)}_n}(t)$, $\boldsymbol{Z^{(2)}_n}(t)$ come from i.i.d. $\boldsymbol{X_1}$ and $\boldsymbol{X_2}$ with distribution determined by $\boldsymbol{Z_0}(t)$. Then
	\begin{align}
	\begin{split}
	\text{a) } & \hat T_{i,0} := n\left[\boldsymbol{Z^{(i)}_n}(t) - \boldsymbol{Z_0}(t)\right]^{T}\boldsymbol{\hat \Sigma_{z}}^{-1}\left[\boldsymbol{Z^{(i)}_n}(t) - \boldsymbol{Z_0}(t)\right]
	\xrightarrow{d} \mathcal{\chi}^2(2) \quad \text{as } n\to \infty, \quad i=1,2 \\
	\text{b) } & \hat T_{1,2} := n\left[\boldsymbol{Z^{(1)}_n}(t) - \boldsymbol{Z^{(2)}_n}(t)\right]^{T}\boldsymbol{\hat \Sigma_{z}}^{-1}\left[\boldsymbol{Z^{(1)}_n}(t) - \boldsymbol{Z^{(2)}_n}(t)\right]
	\xrightarrow{d} \mathcal{\chi}^2(2) \quad \text{as } n\to \infty.
	\end{split}
	\end{align}
\end{theorem}
Appendix \ref{sec:Proofs} includes a proof of this result.

\subsection{Hypothesis testing based on ECFs of colour intensities}
\label{subsc:ParamHypTest}
Now, let us return to the data extraction (Figure \ref{fig:param_flowchart}). Assume we have an image of forest and a test image. Having extracted matrix $\boldsymbol{X}$, one can estimate parameters of the stable distributions of the forest picture channels and compute vectors $\boldsymbol{Z_0}^r(t), \boldsymbol{Z_0}^g(t)$ and $\boldsymbol{Z_0}^b(t)$, where the superscripts $r,g$ and $b$ refer to red, green and blue channels. Similarly, after obtaining $\boldsymbol{X}$ from the test image, empirical characteristic functions are computed and subsequently $\boldsymbol{Z_n}^r(t), \boldsymbol{Z_n}^g(t)$ and $\boldsymbol{Z_n}^b(t)$ are obtained. Note that we use CFs and ECFs of each channel separately, so dependencies between colours are ignored so are spatial dependencies.
\begin{figure}[b!]
	\centering
	\includegraphics[width=0.9\textwidth]{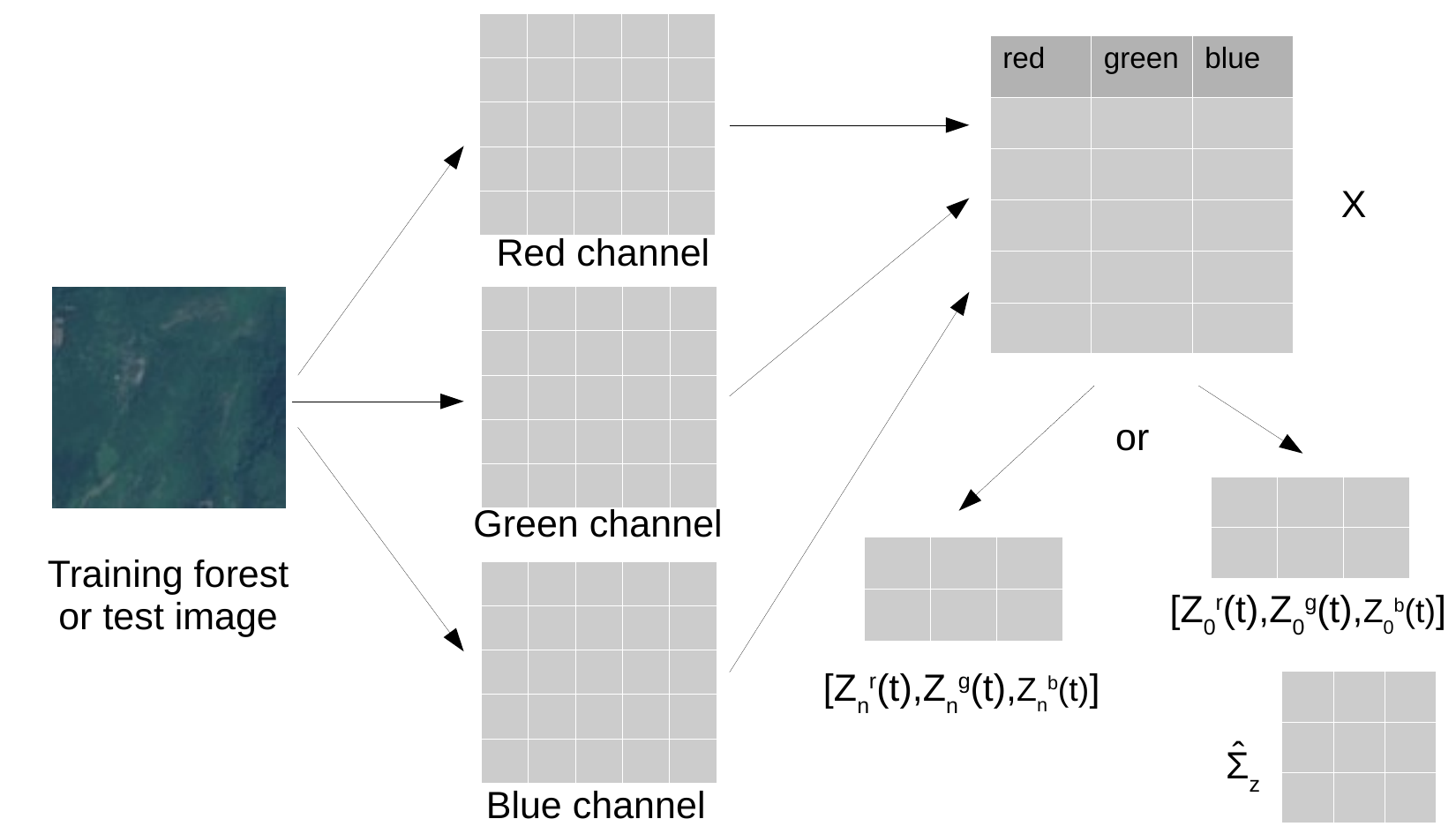}
	\caption{Data processing for the parametric testing}
	\label{fig:param_flowchart}
\end{figure}

Hypothesis testing in the parametric setting is illustrated in Figure \ref{fig:param_mult_hyp_test}. A test image is compared with each of the forest images in the training set. When comparing a test image and a forest one, $[\boldsymbol{Z_0}^r(t), \boldsymbol{Z_0}^g(t), \boldsymbol{Z_0}^b(t)]$ and $[\boldsymbol{Z_n}^r(t), \boldsymbol{Z_n}^g(t), \boldsymbol{Z_n}^b(t)]$ are used to test separately that the colour intensities of the test image are those of forest:
\begin{equation}
\label{eq:par_mult_hyp_test}
\begin{split}
T^r:=n(\boldsymbol{Z_n}^r - \boldsymbol{Z_0}^r)^{T}\boldsymbol{\hat \Sigma_{z;r}}^{-1}(\boldsymbol{Z_n}^r - \boldsymbol{Z_0}^r)
\xrightarrow{d} \mathcal{\chi}^2(2) \\
T^g:=n(\boldsymbol{Z_n}^g - \boldsymbol{Z_0}^g)^{T}\boldsymbol{\hat \Sigma_{z;g}}^{-1}(\boldsymbol{Z_n}^g - \boldsymbol{Z_0}^g)
\xrightarrow{d} \mathcal{\chi}^2(2) \\
T^b:=n(\boldsymbol{Z_n}^b - \boldsymbol{Z_0}^b)^{T}\boldsymbol{\hat \Sigma_{z;b}}^{-1}(\boldsymbol{Z_n}^b - \boldsymbol{Z_0}^b)
\xrightarrow{d} \mathcal{\chi}^2(2) 
\end{split}
\end{equation}
It implies that in order to be tagged as forest, a test image must have all three colour samples passing their corresponding test. Since the limiting distributions in (\ref{eq:par_mult_hyp_test}) are the same, thresholds for each channel are set to be equal.

For each pair of a test image and a forest one, $T_{i}:=\min(T^r_i, T^g_i, T^b_i)$ is computed and then $T_{min}:=\min_i(T_{i})$ is obtained and compared with the threshold.
\begin{figure}[H]
	\centering
	\includegraphics[width=.8\textwidth]{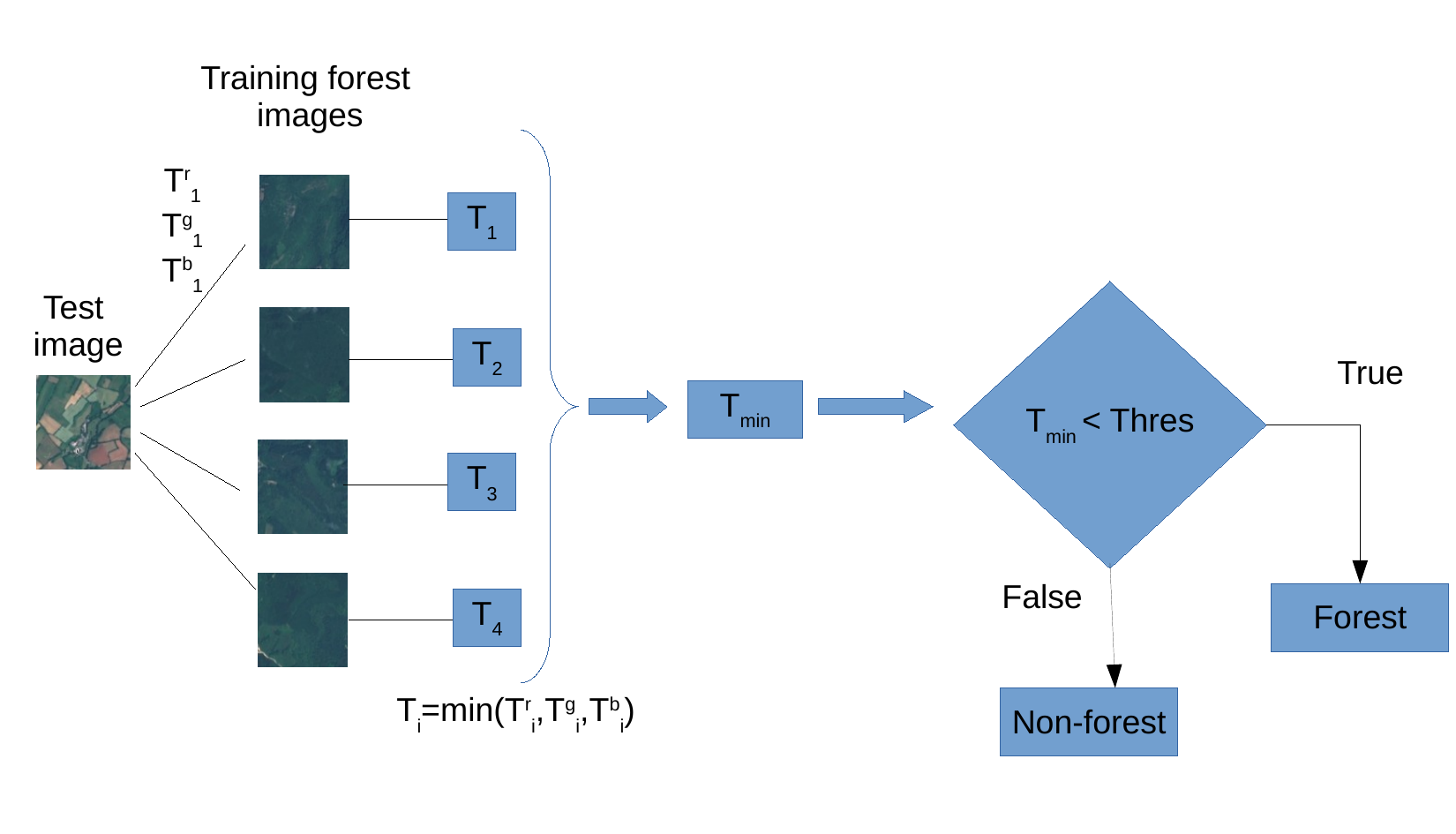}
	\caption{Classification in the parametric model}
	\label{fig:param_mult_hyp_test}
\end{figure}

\subsection{Training and cross-validation of the parametric model}
\label{subsc:par_train_cv}
The training procedure starts with finding $\boldsymbol{\hat \Sigma_{z}}$'s, $\boldsymbol{Z_0}$'s and $\boldsymbol{Z_n}$'s for forest images and $\boldsymbol{Z_n}$'s for non-forest ones in formulas (\ref{eq:par_mult_hyp_test}). After that, training proceeds as in \ref{subsc:nonpar_train_cv}: T statistics for pairs of images are computed and the threshold is set to minimise the misclassification rate. Throughout classification of new images, $\boldsymbol{\hat \Sigma_{z}}$'s and $\boldsymbol{Z_0}$'s corresponding to training forest images will be available in advance, so that only $\boldsymbol{Z_n}$'s of the new images will be left to compute to perform hypothesis testing based on (\ref{eq:par_mult_hyp_test}). For that reason we do not build the parametric framework on Theorem \ref{thm:param_test} b) with $\boldsymbol{\hat \Sigma_{z}}$'s obtained from both the test image and a training forest image: estimation of parameters of stable distribution is quite computationally demanding. Although, this approach would provide better accuracy.

Cross-validation for the parametric model is summarised in Algorithm \ref{alg:p}.

\begin{algorithm}[H]
	\caption{Parametric cross validation procedure
		\label{alg:p}}
	\begin{algorithmic}[1]
		\State \textbf{Inputs:} Data set of images labelled forest or non-forest, number of folds $k$, max threshold $T_{max}$, parameter for CF's ($t$ in (\ref{stable def2})).
		\State For every forest image compute $Z_0$, parameters and the covariance matrix.
		\State For every non-forest image compute $Z_n$.
		\State Randomly shuffle images and group them into $k$ folds.
		\For{$i =1$ to $k$}
		\For{each image in fold $i$}    
		\State \parbox[t]{313pt}{Compute 3 $\hat T$'s from (\ref{eq:par_mult_hyp_test}), between the image and all images labelled forest in the remaining $k-1$ folds.\strut}
		\State Keep only the smallest $\hat T$ for the image.
		\EndFor
		\EndFor
		\For{$t$ in $0$ to $T_{max}$}
		\State \parbox[t]{313pt}{For all images and all $\hat T$'s, tag images with $\hat T<t$ as forest and $\hat T>t$ as non forest. \strut}
		\State Compute the accuracy for the given threshold $t,$ $\text{Acc}(t).$ 
		\EndFor
		\State \Return $\text{arg}\max_{t}\text{Acc}(t)$ and $Z_0$'s, parameters and the covariance matrices of forest images.
	\end{algorithmic}
\end{algorithm}

\section{Comparison of the methods}
\label{Sec:TheorCompar}
There are features typical for both of our methods. First, if by accident there is a non-forest image in the training forest set, test images of its type start to be tagged as forest because they pass the corresponding test. As a consequence, additional false positive results will occur. Specially designed cross-validation may be exploited to eliminate this problem. Second, the setting is multiple hypothesis testing, so it could naturally be re-set to classify forests images into multiple types such as rain forest, coniferous forest, etc. Another thing is that clouds and sunbeam scattering affect the algorithms because they introduce noise to the distributions of objects. It could be cured by atmospheric correction (\cite{sen2cor17}). Also, both models increase accuracy with growth of the numbers of points in the training data and test data. Improvement of prediction accuracy could be achieved in three ways. Increasing resolution, e.g. by switching from satellite data to aerial imagery, would add more pixels to objects under consideration while adding new tagged images to the training set would provide better training of thresholds and more tagged classes. In practice, merely increasing the size of an image with fixed resolution (adding more area) works well only in the case of training pictures: enlarging test images would imply observing pictures with objects belonging to different classes and tagging some areas incorrectly because of that. And lastly, unlike classical machine learning methods, the models developed in the present article are robust to data transformations such as rotation and reflection.

There is also a number of differences between the two models. An advantage of the parametric model is that it can work on incomplete data: e.g. when only 1 channel is available or different channels are available at different times. The non-parametric method needs at least two channels simultaneously. On the other hand, the disadvantage of the parametric model is that the algorithm generally requires more pixels per image in training as it involves parameter estimation and because we chose to use Theorem \ref{thm:param_test} a). The estimation procedure also makes the parametric method slower than the other one. A con of the non-parametric method is that the model does not take into account the actual distributions of colour intensities, which can contribute to additional false positives.

To summarise the comparison, we point out that at the current stage of research it is not possible to claim superiority of one method over the other as they both have pros and cons.

\section{Empirical application}
\label{sec:emp_app}
Sentinel-2 satellite images are used to illustrate how the new Mahalanobis distance classifier (MDC) and  stable distribution classifier (SDC) perform in comparison to some other commonly used classification methods. The methods used for comparison are support vector machine (SVM) and convolutional neural network (CNN). Both of the comparison methods have been widely used for classification, particularly in the setting of vegetation remote sensing \citep{kattenborn2021review, maulik2017remote} .

Three experiments with real data are shown here: two comparing accuracy of the methods and one giving visual representation of performance. As the source of terrain images, the region Occitanie in France is used, more precisely, vicinity of the city of Albi.

In the first two experiments, we look at how the methods perform using a training and test split of the available dataset. A training set consisting of 510 images of sizes $20 \times 20$ pixels with an equal number forest and non-forest images is used in the first experiment. The test set is also equally balanced between forest and non-forest images but it consists of 200 images of sizes $20 \times 20$ pixels. The second experiment is similar, except that each $20 \times 20$ image is split into four equal pieces, making 2040 images of sizes $10 \times 10$ pixels in the training set and 1000 $10 \times 10$ pixels images in the test set. 

In the third experiment, we show how the methods behave when every sub-image of size $10 \times 10$ pixels is classified of an image of size $2000 \times 2000$ pixels. This makes it possible to see how the methods behave for different terrains and get a visual interpretation of the results. 

In each of the experiments, the SVM is implemented with the radial basis function kernel which is a commonly used kernel for classification \cite{chang2010training}. The feature space consists of all pixels and their RGB colour values in the form of flatten input. The Python library 'sklearn' and its default settings are used for the implementation. 

The CNN is implemented with 'tensorflow' in Python and the network consists of 2 convolution layers with 8 filters in each layer and a kernel size of 3. ReLu is used as activation function in the convolution layers together with maximum pooling of size $2 \times 2$. A dense layer with sigmoid as activation function is used in the final layer of the network. The loss function is set to be 'binary\_crossentropy' and 'Adam' is the solver of choice. It should be mentioned that we tried and assessed CNNs with all possible combinations of \{1,2,3\} convolution layers, \{8,12,16\} filters, \{2,3\} kernel size, \{0, 0.05, 0.1\} dropout. The one described above was the best performing in terms of accuracy.

The performance measure used to evaluate the methods in the first two experiments is the error rate, i.e., the number of incorrectly classified images in the test set divided by the total number of test images. Figure \ref{fig:all_results} illustrates how the methods perform when the training set varies between 10 \% and 100 \% of the total training set using images of sizes $20 \times 20$ pixels and $10 \times 10$ pixels, respectively. Even though the training set size varies, an equal number of forest and non-forest images are always used for training and the test set is kept unchanged. It should also be noted that all of the images are collected from the same region and under similar conditions. This means that the models are not trained for different seasons or different types of forests. However, it would be possible to make the classification models more general by including more diverse data.

\begin{figure}[h]
	\subfloat[$20 \times 20$ pixels]{\includegraphics[width= 0.49\textwidth]{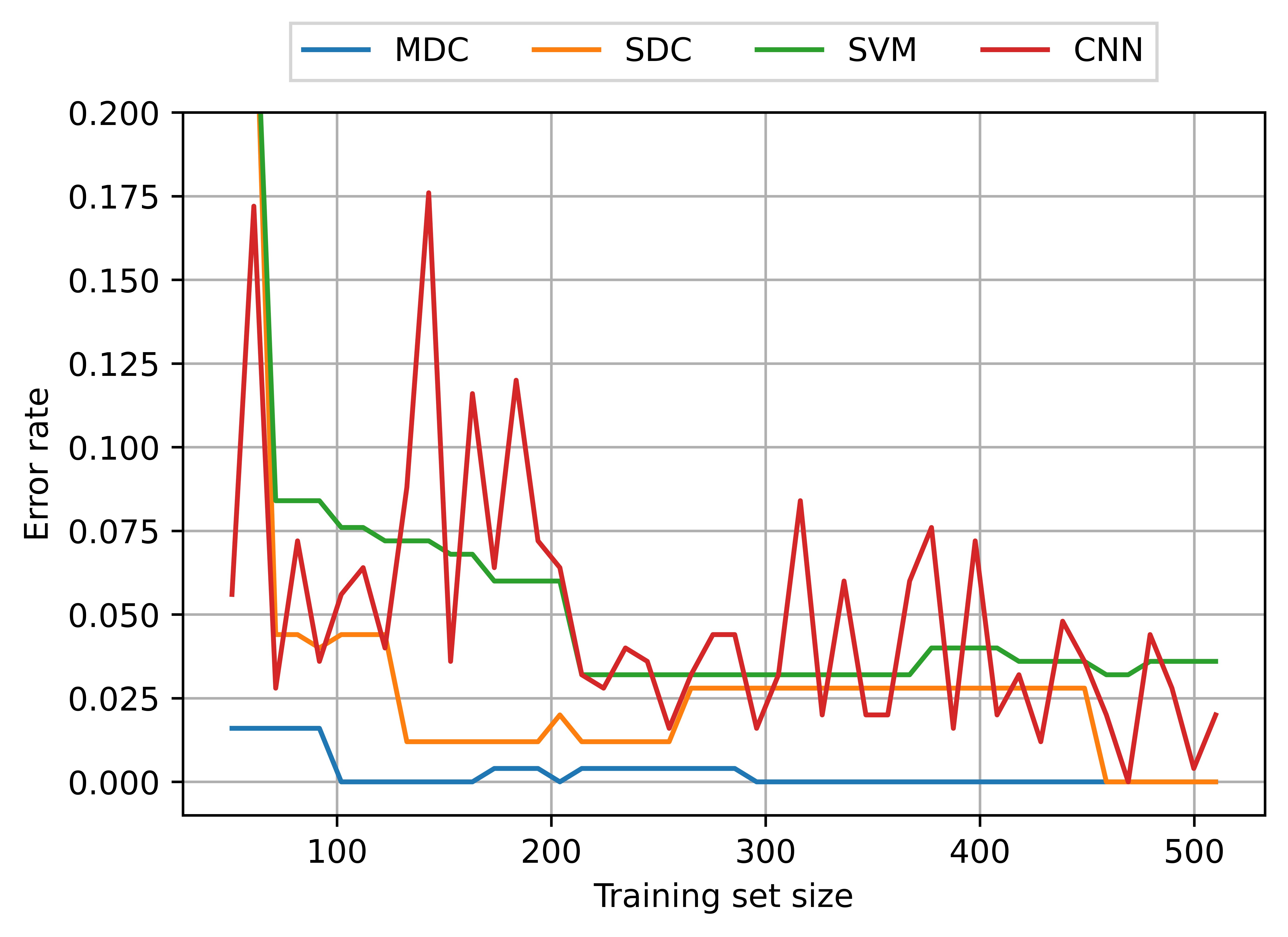}}
	\label{fig:all_results_M20_N20}
	\subfloat[$10 \times 10$ pixels]{\includegraphics[width= 0.49\textwidth]{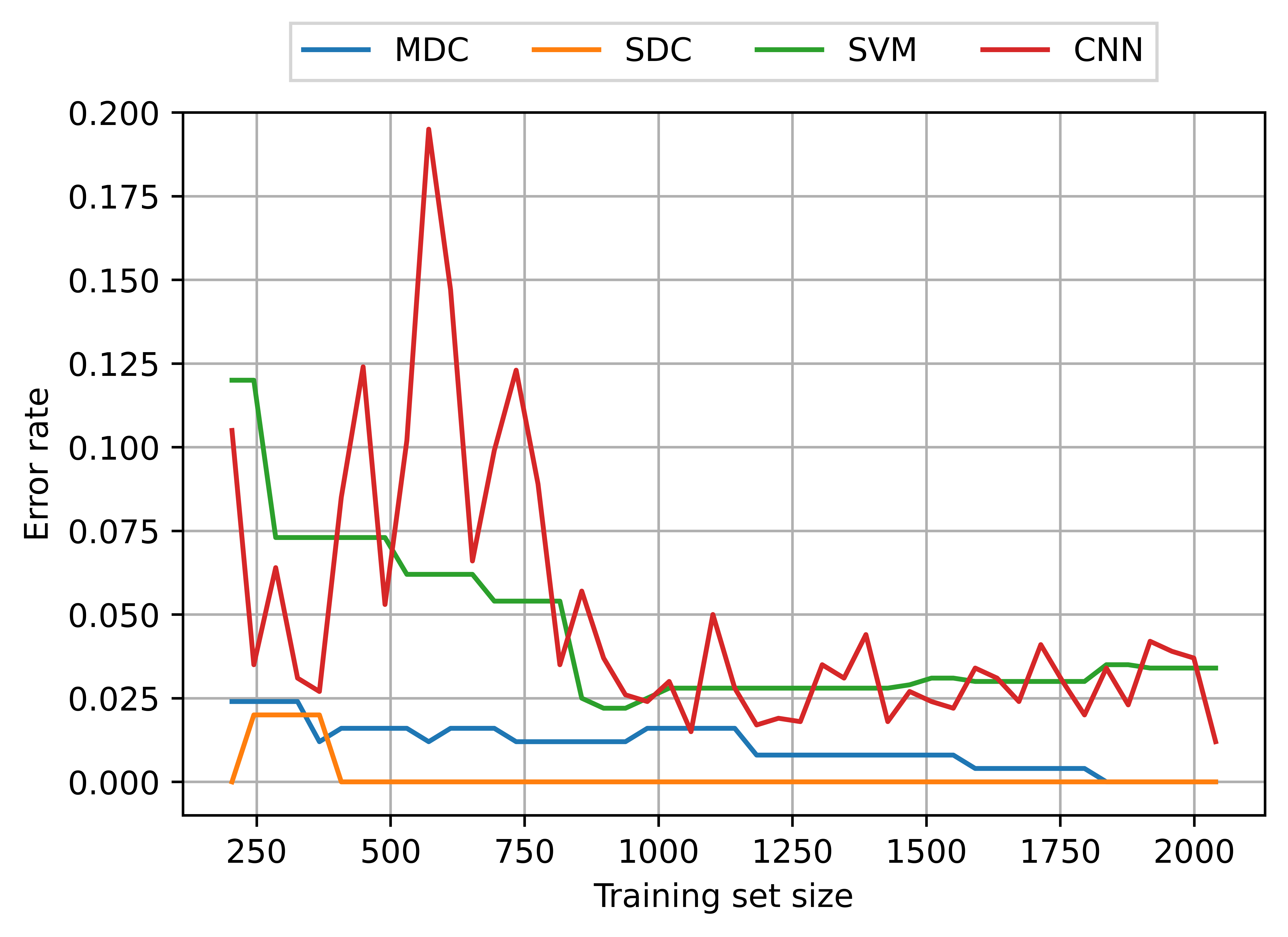}}	
	\label{fig:all_results_M10_N10}
	\caption{Forest classification error rates depending on the training set size using MDC, SDC, SVM and CNN for different image sizes.}
	\label{fig:all_results}
\end{figure}

Figure \ref{fig:all_results} shows that the new methods perform better than both of the comparison methods. In the $20 \times 20$ pixels case, the MDC generally has a lower error rate than the SDC. The error rate corresponding to the MDC drops to zero when the size of the training set is about 300 whereas it takes until about 450 for the SDC. The SVM and CNN perform quite similarly, with possibly a slight advantage to the CNN. However, the SVM is more robust to changes in performance when adding more data to the training set. Neither of the comparison methods tends to have an error rate below 1 \%. Analogous results are presented in the $10 \times 10$ pixels case. However, in this experiment the SDC usually has a lower error rate than the MDC. The error rate drops to zero almost immediately when using the SDC whereas the error rate corresponding to the MDC gradually decreases from 2.4 \% and reaches zero when the training set size is about 1800. Again, the SVM and CNN show similar (worse) performance.

Next we turn to the third experiment and visually investigate the methods by classifying $10 \times 10$ pixels sub-images of a $2000 \times 2000$ pixels image. The same (complete) training set as in the previous experiments is used to calibrate the methods. This training set does not contain images that are part of the large image but they were collected from a nearby region under similar conditions. Sub-images that are classified as forest are marked as white and non-forest sub-images are marked as black. The result of the classification is illustrated in Figure \ref{fig:black_white}.

\begin{figure}[H]
\begin{minipage}{.27\textwidth}
\begin{subfigure}{\textwidth}
\includegraphics[width=\textwidth]{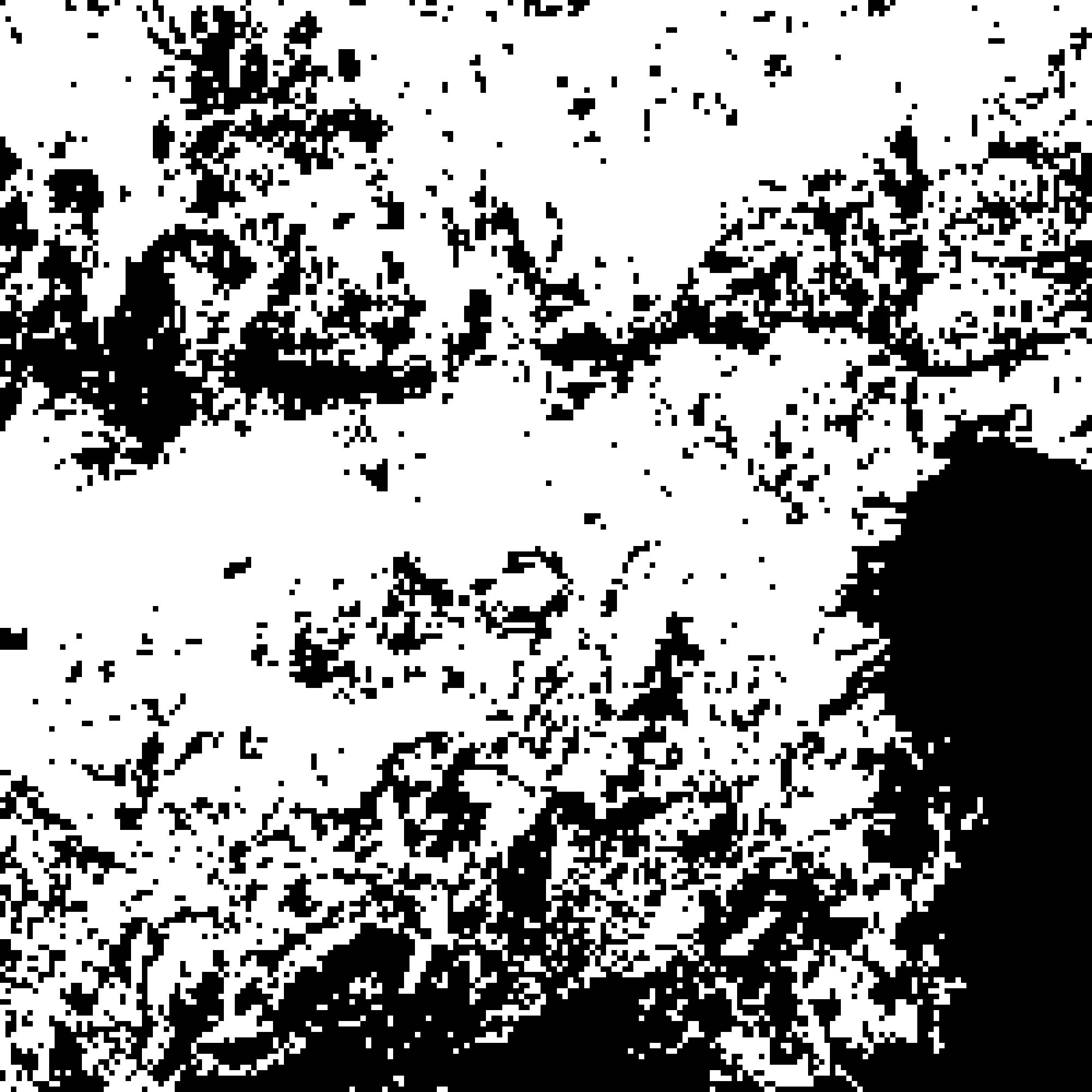}
\caption{MDC}
\end{subfigure}
\begin{subfigure}{\textwidth}
\includegraphics[width=\textwidth]{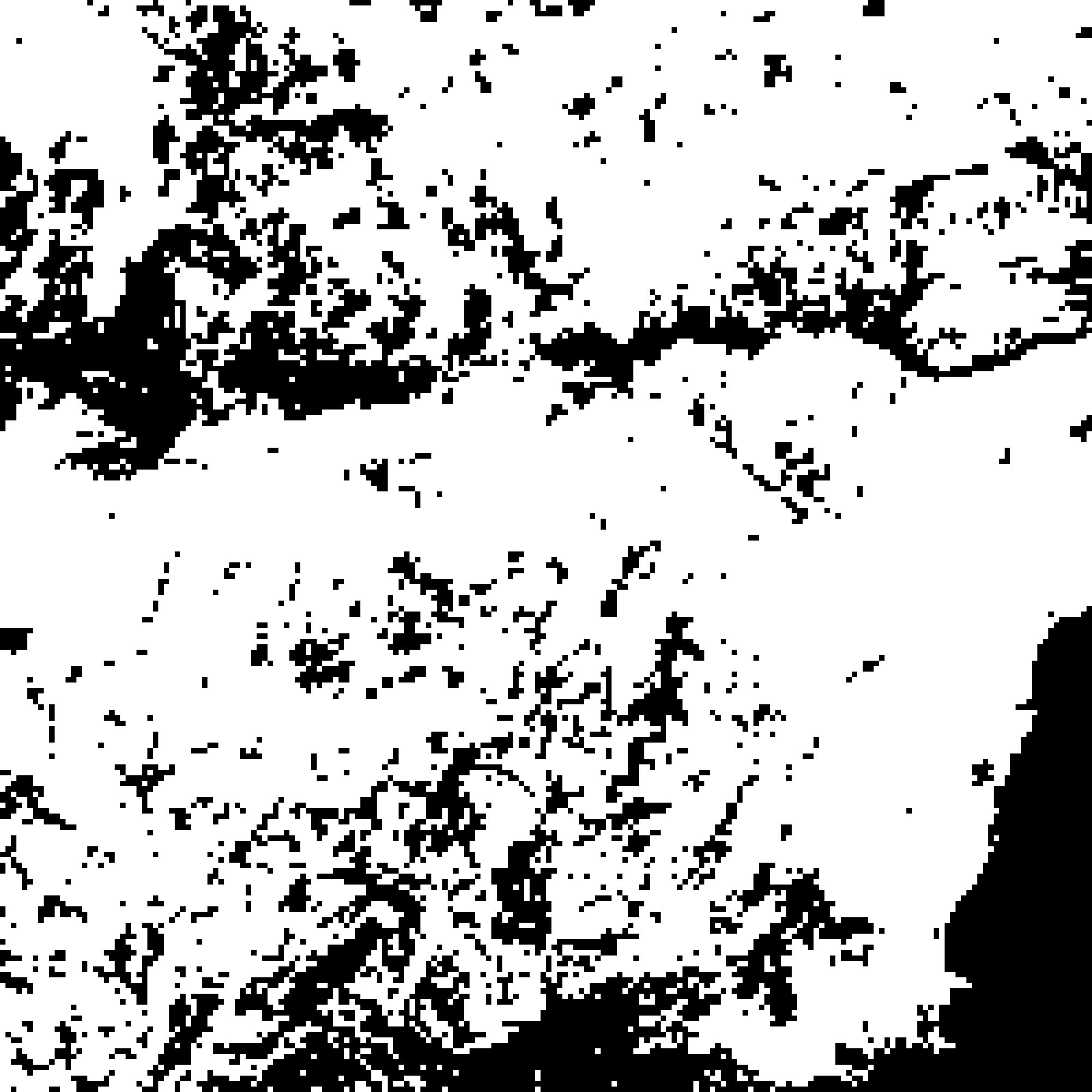}
\caption{SVM}
\end{subfigure}
\end{minipage}
\hfill
\begin{minipage}{.27\textwidth}
\begin{subfigure}{\textwidth}
\includegraphics[width=\textwidth]{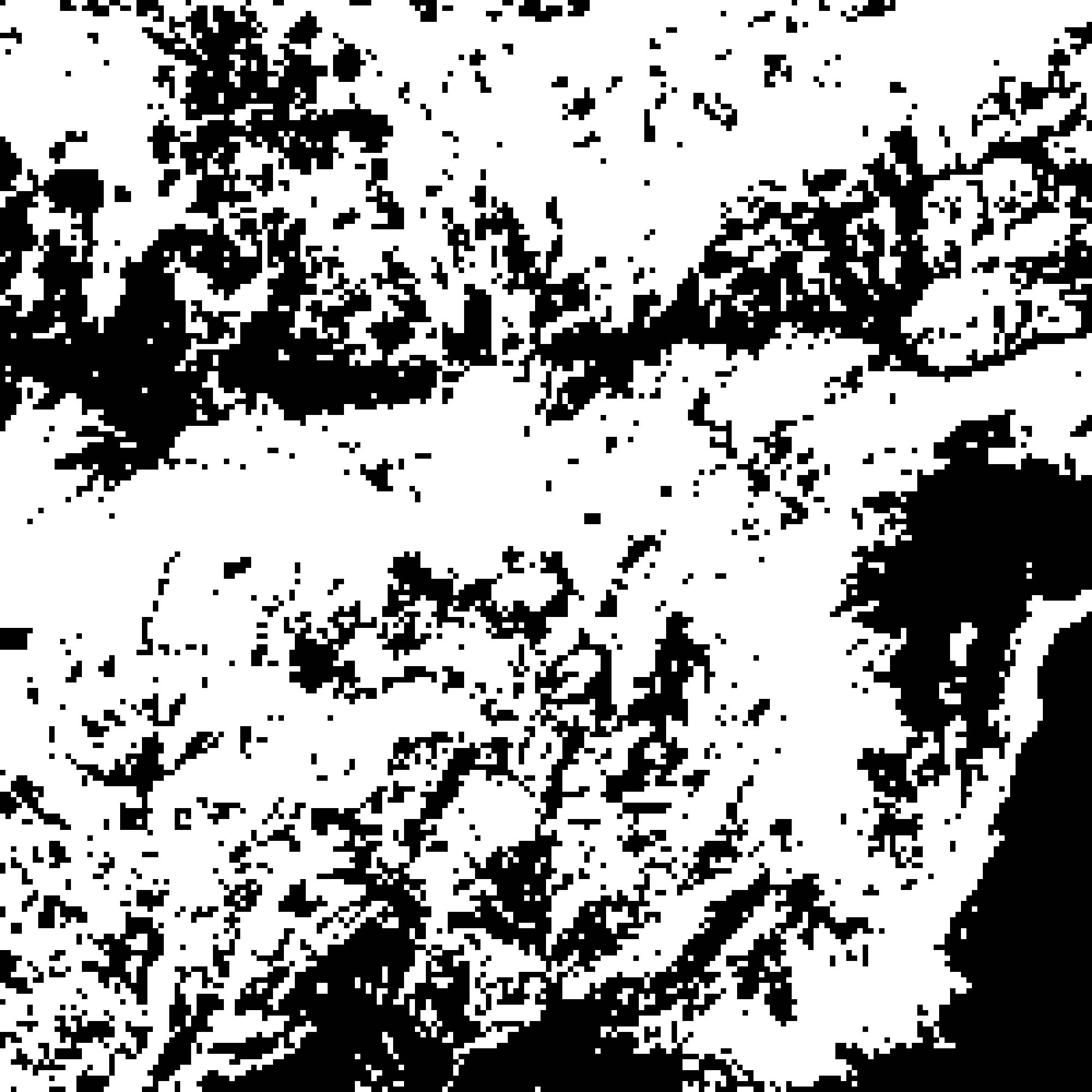}
\caption{SDC}
\end{subfigure}
\begin{subfigure}{\textwidth}
\includegraphics[width=\textwidth]{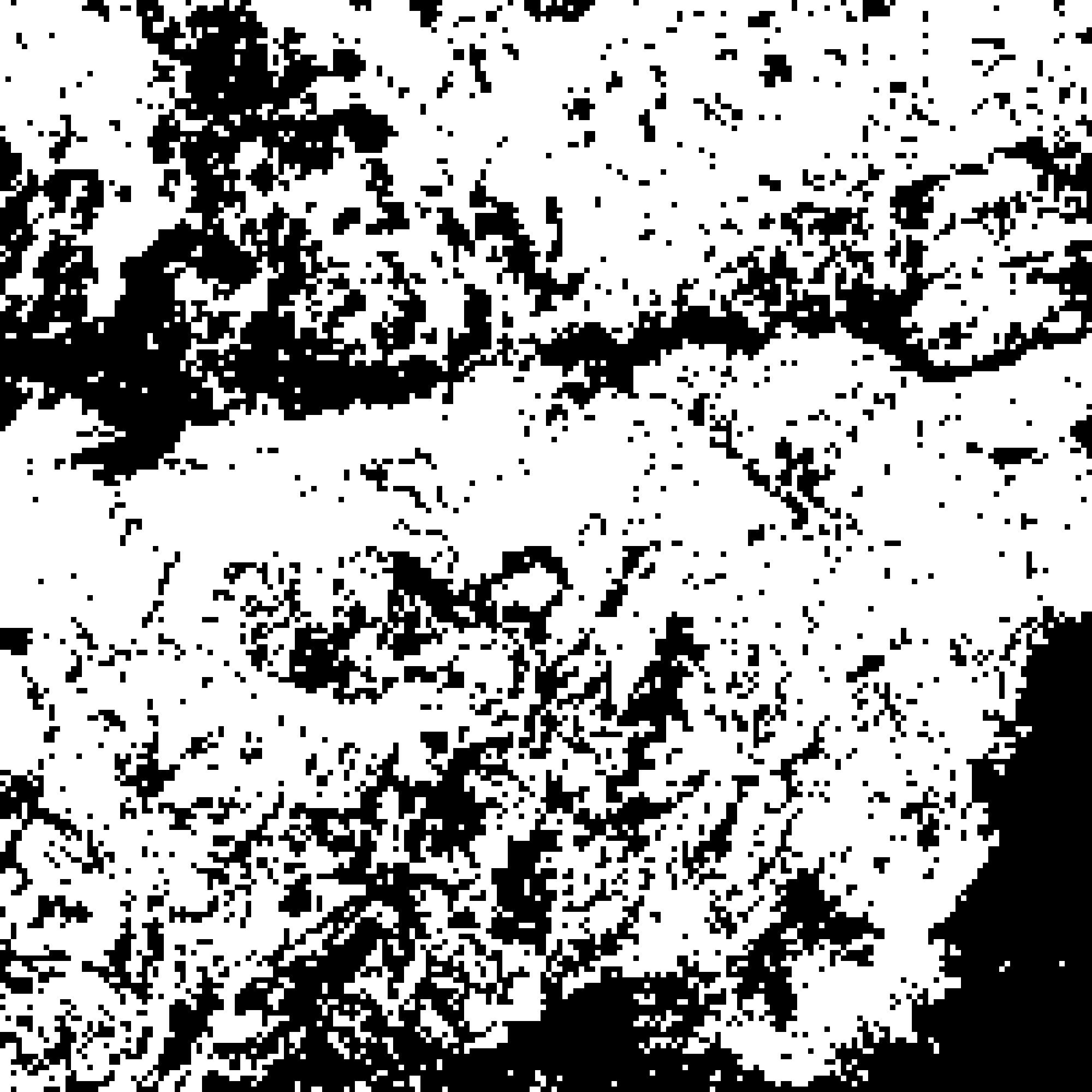}
\caption{CNN}
\end{subfigure}
\end{minipage}
\hfill
\begin{minipage}{.4\textwidth}
\begin{subfigure}{\textwidth}
\includegraphics[width=\textwidth]{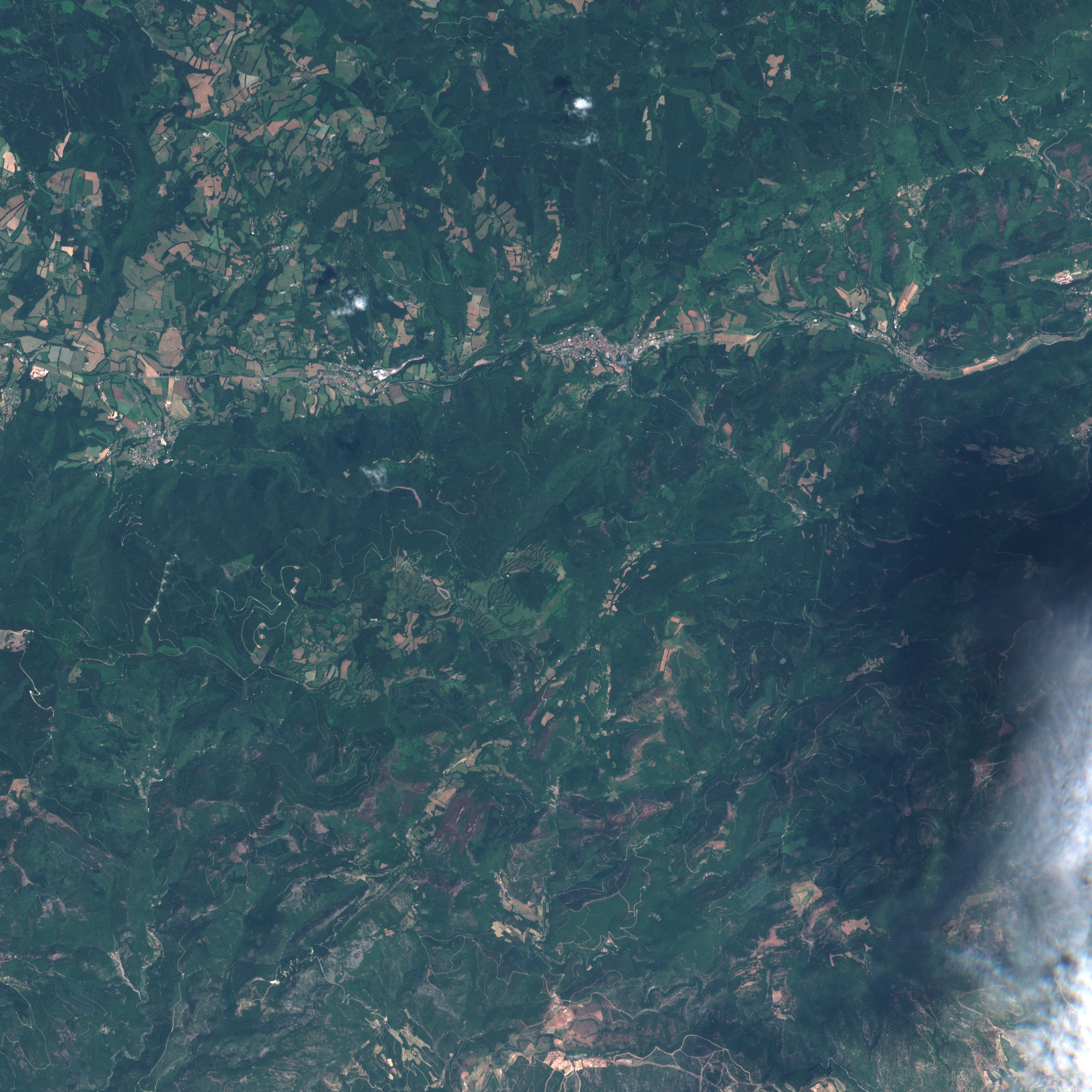}
\subcaption{Original}
\end{subfigure}
\end{minipage}
\caption{Forest classification using MDC, SDC, SVM and CNN of an $2000 \times 2000$ pixels image using sub-images of sizes $10 \times 10$ pixels. White corresponds to forest and black corresponds to non-forest.}
\label{fig:black_white}
\end{figure}

Generally, all four models performed well, though there are differences in their outcomes. One can see that unlike MDC and SDC, SVM and CNN are not so sensitive to the intensity of the green colour. It is best illustrated by performances of the models on forest covered by a shadow of the cloud in the right part of Figure \ref{fig:black_white}. MDC and SDC incorrectly treat the shadowed vegetation as non-forest, while the other two tag it as forest. Shades from clouds generally could be accounted for by cloud detection, which is a solved problem in Earth image processing (see e.g. \cite{sen2cor17}). 


\section{Conclusion}
Forest images can be successfully modeled using alpha-stable distributions. This fact was used to create two new frequentist approaches to forest/non-forest classification. In the experiments with real data our methods outperform the benchmark machine learning ones, which indicates superiority of the former methods on small data. The same was not checked for big data and we note that frequentist techniques tend to be inferior to machine learning in such cases.

The new algorithms are very sensitive to colour intensities, which means that the training data must be obtained under the same conditions as the data for classification. Given this, on a dataset of images collected from different sources under different conditions, machine learning algorithms have an advantage.

Despite the fact that the new algorithms are used for 0-1 classification, they are suited for multi-type classification tasks such as detecting different types of forest.


\bibliography{bibliography}


\begin{appendices}
\section{Data pre-processing technology}
\label{sec:Data_pre-processing}
We have built a preprocessing procedure of creating RGB composites from Sentinel satellite images using Python programming language.

First, we download Sentinel 2 Level-1C products from PEPS (\cite{PEPS}) with the means of \href{https://eodag.readthedocs.io/en/stable/}{EODAG}. The PEPS (Plateforme d’Exploitation des Produits Sentinel) platform redistributes the products of Sentinel satellites, S1A, S1B, S2A and S2B, S3A and S3B from COPERNICUS, the European system for Earth monitoring. EODAG (Earth Observation Data Access Gateway) is a command line tool and a Python package for searching and downloading remotely sensed images while offering a unified API for data access regardless of the data provider.

Second, for each product we do a preprocessing procedure of creating RGB composites from Sentinel 2 Level-1C products. Using \href{https://rasterio.readthedocs.io/en/latest/}{rasterio} package we extract red, green and blue bands from channels B04, B03 and B02 respectively and combine them in a single file using \href{https://gdal.org/drivers/raster/gtiff.html\#raster-gtiff}{GeoTIFF} file format. 

Third, we compress RGB-composites with lossless compression algorithm to decrease the amount of storage space needed for data using gdat\_translate utility (\cite{gdal}).

Lastly, in order to obtain images with sufficient brightness, we normalize values in the channels to range from 0.0 to 1.0 by dividing them by the constant 2000. We truncate values with normalised intensity more than 1.0 by setting them to 1.0, which primarily affects bright white clouds and conforms with typical "crop and scale" image processing procedure. Deforestation detection algorithms use these normalised values.

\section{Distributions of colour intensities of other natural objects}
\label{sec:EDA_adds}
In this section we explore two RGB images: one of a mountain and another one---of sea. Each picture is accompanied by empirical density functions plotted for all channels. For comparison, best fitting stable, gamma and normal distributions are shown.
\begin{figure}[h!]
	\subfloat[Mountain image analysed]{\includegraphics[width= 0.48\textwidth]{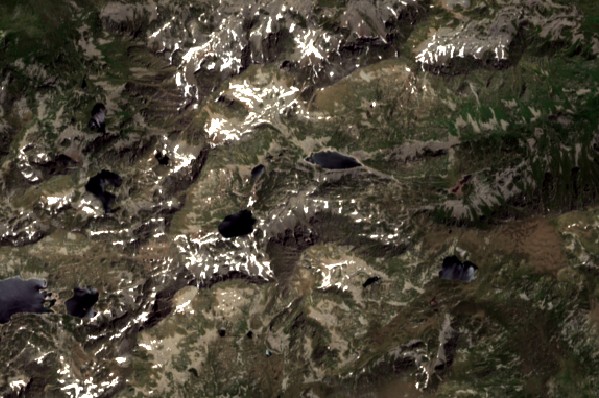}}
	\subfloat[Comparison of densities for red color intensity]{\includegraphics[width= 0.48\textwidth]{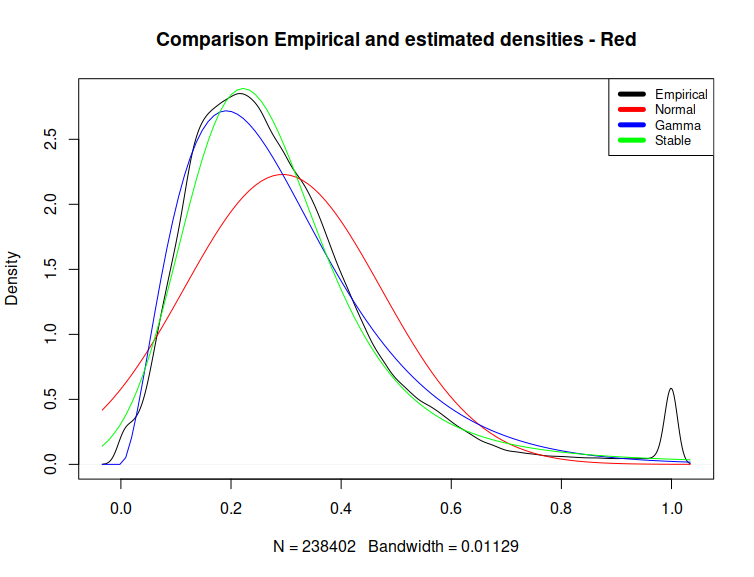}}\\
	\subfloat[Comparison of densities for green color intensity]{\includegraphics[width=0.48\textwidth]{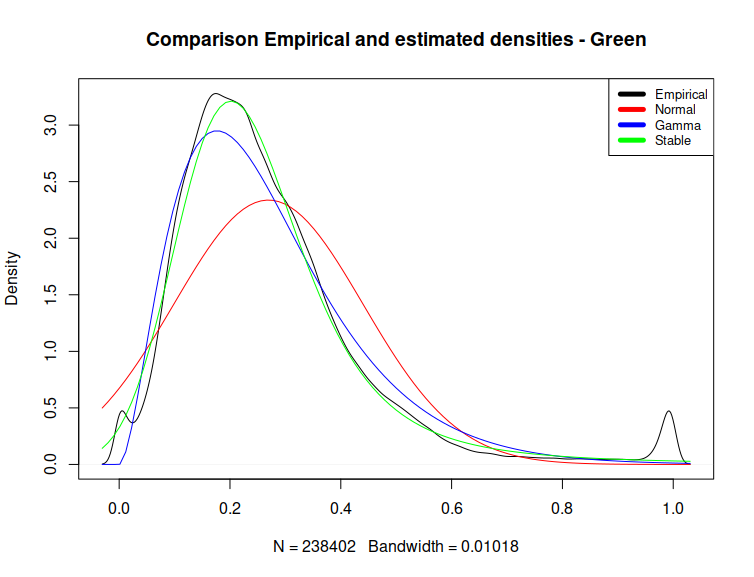}}
	\subfloat[Comparison of densities for blue color intensity]{\includegraphics[width= 0.48\textwidth]{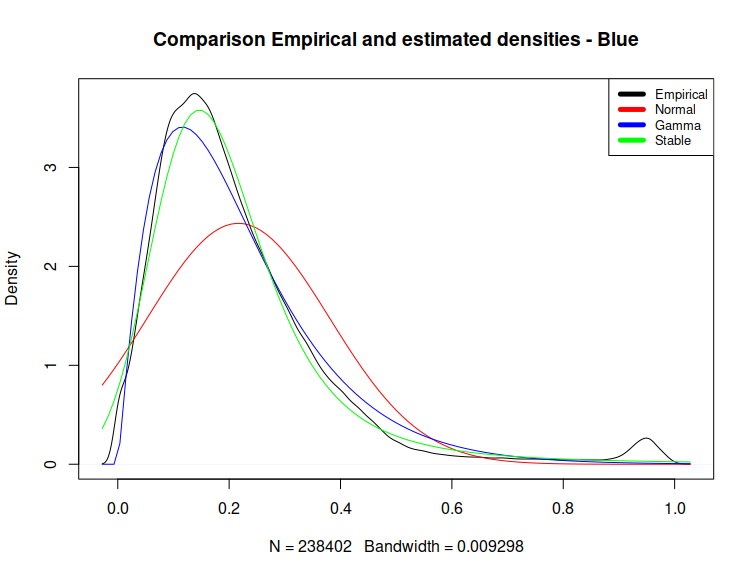}}
	
	\caption{Plots of empirical and estimated probability density functions for distribution of colour intensities in an image of mountain.}
	\label{fig:Mountain}
\end{figure}

\begin{figure}[H]
	\subfloat[Sea image analysed]{\includegraphics[width= 0.48\textwidth]{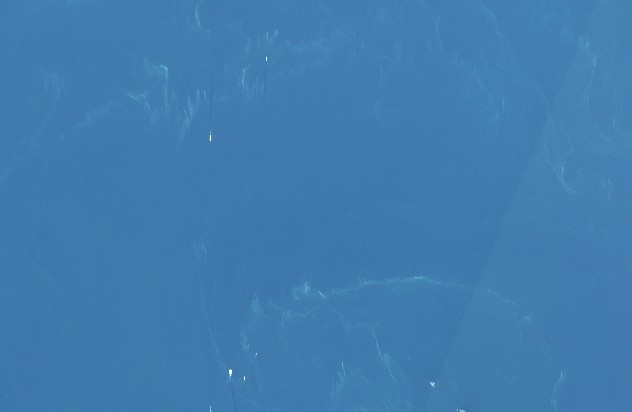}}
	\subfloat[Comparison of densities for red color intensity]{\includegraphics[width= 0.48\textwidth]{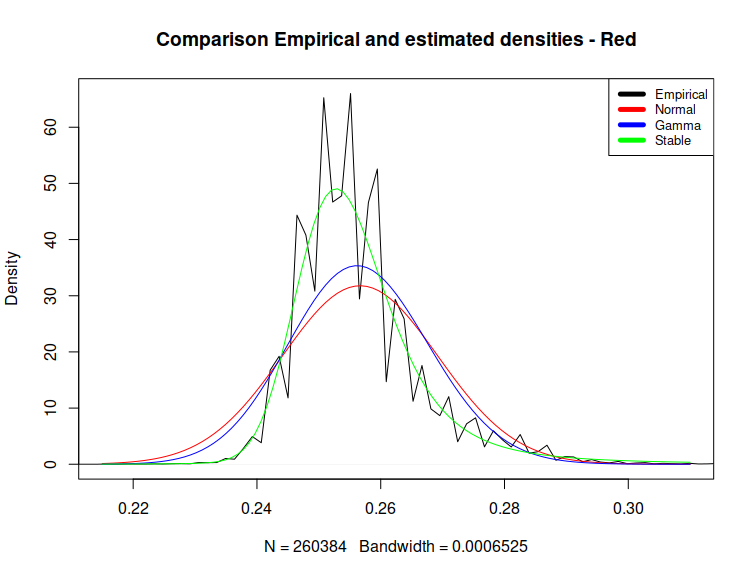}}\\
	\subfloat[Comparison of densities for green color intensity]{\includegraphics[width=0.48\textwidth]{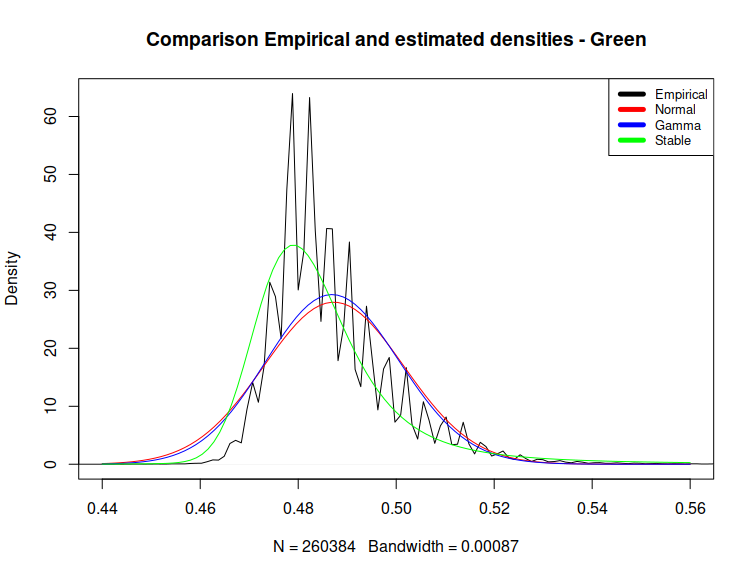}}
	\subfloat[Comparison of densities for blue color intensity]{\includegraphics[width= 0.48\textwidth]{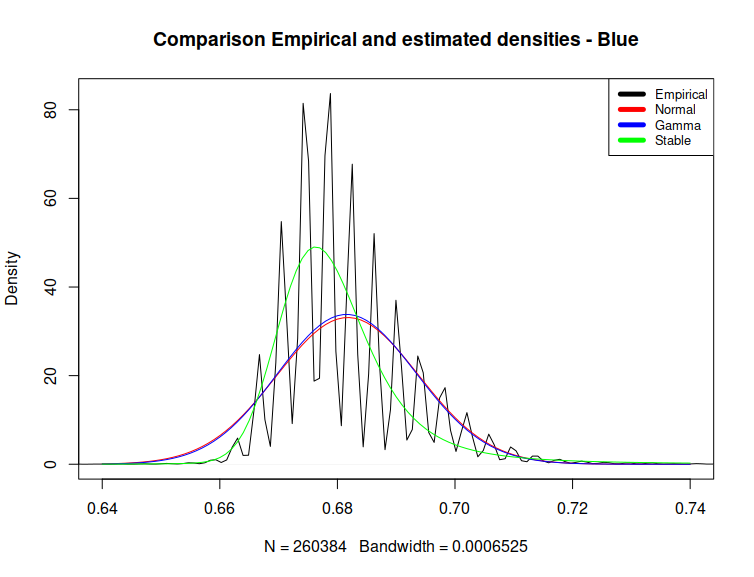}}
	
	\caption{Plots of empirical and estimated probability density functions for distribution of colour intensities in an image of sea.}
	\label{fig:Sea}
\end{figure}

\section{Proofs}
\label{sec:Proofs}
\begin{proof}[Proof of Theorem \ref{thm:NonPar_test}.]
	By the multidimensional central limit theorem
	\begin{equation}
	(\boldsymbol{\bar X_1} - \boldsymbol{\bar X_2}) \xrightarrow{d} \mathcal{N}(\boldsymbol{0}, \gamma \boldsymbol{\Sigma}), \quad \text{where } \gamma=n_1^{-1}  + n_2^{-1}.
	\end{equation}
	After introducing $\boldsymbol{A} := (\gamma \boldsymbol{\Sigma})^{1/2}$ and $\boldsymbol{\hat A} := (\gamma \boldsymbol{\hat \Sigma})^{1/2}$ we can write
	\begin{equation}
	\label{eq:Thm1_X_diff_converg}
	(\boldsymbol{\bar X_1} - \boldsymbol{\bar X_2}) \xrightarrow{d} \boldsymbol{A} \boldsymbol{Z}, \quad \text{where } \boldsymbol{Z} \sim \mathcal{N}(\boldsymbol{0}, \boldsymbol{I}_p).     
	\end{equation}
	Taking into account (\ref{eq:Thm1_X_diff_converg}) and that $\boldsymbol{\hat A}^{-1} \xrightarrow{P} \boldsymbol{A}^{-1}$, where the later is a matrix of constants, by a property of weak convergence
	\begin{equation}
	\label{eq:joint_converg}
	\left[\boldsymbol{\hat A}^{-1}, (\boldsymbol{\bar X_1} - \boldsymbol{\bar X_2})\right]
	\xrightarrow{d}  
	\left[\boldsymbol{A}^{-1}, \boldsymbol{AZ} \right].
	\end{equation}
	Then, we introduce a new quantity
	\begin{equation}
	\label{eq:Q_def}
	\begin{split}
	\boldsymbol{Q}:=
	\left[\boldsymbol{\hat A}^{-1} (\boldsymbol{\bar X_1} - \boldsymbol{\bar X_2})\right]^T 
	\left[\boldsymbol{\hat A}^{-1} (\boldsymbol{\bar X_1} - \boldsymbol{\bar X_2})\right].
	\end{split}
	\end{equation}
	The multivariate continuous mapping theorem applied to (\ref{eq:joint_converg}) implies $\boldsymbol{Q} \xrightarrow{d} \boldsymbol{Z}^T \boldsymbol{Z} \stackrel{d}{=} \chi^2(p)$.
	On the other hand,
	\begin{equation}
	\boldsymbol{Q} = (\boldsymbol{\bar X_1} - \boldsymbol{\bar X_2})^T (\boldsymbol{\hat A}^{-1})^T \boldsymbol{\hat A}^{-1} (\boldsymbol{\bar X_1} - \boldsymbol{\bar X_2}),
	\end{equation}
	and using the following chain of transformations 
	\begin{equation}
	(\boldsymbol{\hat A}^{-1})^T \boldsymbol{\hat A}^{-1} = (\boldsymbol{\hat A}^T)^{-1} \boldsymbol{\hat A}^{-1} = (\boldsymbol{\hat A}^T \boldsymbol{\hat A})^{-1} = (\gamma \boldsymbol{\hat \Sigma})^{-1}
	\end{equation}
	we obtain 
	\begin{equation}
	\label{eq:D_sq_fin_form}
	\boldsymbol{Q} = (\boldsymbol{\bar X_1} - \boldsymbol{\bar X_2})^T (\gamma \boldsymbol{\hat \Sigma})^{-1} (\boldsymbol{\bar X_1} - \boldsymbol{\bar X_2}).
	\end{equation}
\end{proof}

\begin{proof}[Proof of proposition \ref{prop:z_n}]
	First, we write  
	\begin{equation}
	\begin{split}
	u_j := \cos(tX_j) = \frac{1}{2} (e^{itX_j} + e^{-itX_j}) \\
	v_j := \sin(tX_j) = \frac{1}{2i} (e^{itX_j} - e^{-itX_j})
	\end{split}
	\end{equation}
	Since both of them are bounded, then $[u_j(t), v_j(t)]^T$ is a random vector with components having the second moment. Thus, the multidimensional central limit theorem immediately applies to this vector. Therefore, it is only left to find the entries of the matrix $\boldsymbol{\Sigma_z}$ (\ref{mtrx:Sigma_z}).
	\begin{equation}
	\label{eq:cov_uv}
	\begin{split}
	\operatorname{Cov}(u_j,v_j)=\operatorname{Cov}\left[ \frac{1}{2}(e^{itX_j} + e^{-itX_j}), \frac{1}{2i} (e^{itX_j} - e^{-itX_j}) \right] &= \\
	\frac{1}{4i} \left[ \operatorname{Cov}(e^{itX_j}, e^{itX_j}) -  \operatorname{Cov}(e^{itX_j}, e^{-itX_j}) + \operatorname{Cov}(e^{itX_j}, e^{-itX_j}) - \operatorname{Cov}(e^{-itX_j}, e^{-itX_j})\right] &= \\
	\frac{1}{4i} \left[ \operatorname{Cov}(e^{itX_j}, e^{itX_j}) - \operatorname{Cov}(e^{-itX_j}, e^{-itX_j})\right]&
	\end{split}
	\end{equation}
	\begin{equation}
	\label{eq:vars_exps}
	\begin{split}
	\operatorname{Var}\left(e^{itX_j}\right) &=\mathrm{E}\left[e^{2itX_j}\right] - \left(\mathrm{E}\left[e^{itX_j}\right]\right)^2 = \phi(2t) - [\phi(t)]^2 \\
	\operatorname{Var}\left(e^{-itX_j}\right)&= \phi(-2t) - [\phi(-t)]^2
	\end{split}
	\end{equation}
	Having substituted covariances in (\ref{eq:cov_uv}) using (\ref{eq:vars_exps}) one has
	\begin{equation}
	\begin{split}
	\operatorname{Cov}(u_j,v_j)= \frac{1}{4i} \left[ \phi(2t) - [\phi(t)]^2 - \phi(-2t) + [\phi(-t)]^2 \right] \\
	\end{split}
	\end{equation}
	Other elements of the covariance matrix $\operatorname{Var}(u_j)$ and $\operatorname{Var}(v_j)$ are computed in the same manner:
	\begin{equation}
	\label{eq:var_u}
	\begin{split}
	\operatorname{Var}(u_j) = \mathrm{E}[u_j^2] - \left(\mathrm{E}[u_j]\right)^2 & = \\
	\frac{1}{4} \mathrm{E} \left( e^{2itX_j} + 2e^0 + e^{-2itX_j} \right) - \frac{1}{4} \left[\mathrm{E} \left( e^{itX_j} + e^{-itX_j} \right)\right]^2 &= \\
	\frac{1}{4} \left[ \phi(2t) + 2 + \phi(-2t) - (\phi(t))^2 -2\phi(t)\phi(-t) - (\phi(-t))^2 \right]&
	\end{split}
	\end{equation}
	\begin{equation}
	\label{eq:var_v}
	\begin{split}
	\operatorname{Var}(v_j) = \mathrm{E}[v_j^2] - \left(\mathrm{E}[v_j]\right)^2 & = \\
	- \frac{1}{4} \mathrm{E}\left[e^{2itX_j} - 2e^0 + e^{-2itX_j}\right] - \left(\frac{1}{2i}\mathrm{E}\left[e^{itX_j} - e^{-itX_j}\right]\right)^2 & = \\
	-\frac{1}{4} \left[ \phi(2t) - 2 + \phi(-2t)\right] - \frac{-1}{4}\left[ \phi(t) - \phi(-t) \right]^2 & = \\
	\frac{1}{4} \left[ [\phi(t)]^2 - 2\phi(t)\phi(-t) +\phi(-t)^2 \right] - \frac{1}{4}[\phi(2t) + \phi(-2t) - 2]&
	\end{split}
	\end{equation}
\end{proof}

\begin{proof}[Proof of Theorem \ref{thm:param_test}]
	a) According to Proposition \ref{prop:z_n} and properties of weak convergence 
	\begin{equation}
	\left(\frac{1}{n}\boldsymbol{\hat \Sigma_{z}}\right)^{-1/2} (\boldsymbol{Z_n} - \boldsymbol{Z_0}) \xrightarrow{d} \mathcal{N}(\boldsymbol{0},\boldsymbol{I}).
	\end{equation}	
	Hence, by continuous mapping theorem,
	\begin{equation}
	\begin{split}
	\left[\left(\frac{1}{n}\boldsymbol{\hat \Sigma_{z}}\right)^{-1/2} (\boldsymbol{Z_n} - \boldsymbol{Z_0})\right]^T  
	\left[\left(\frac{1}{n}\boldsymbol{\hat \Sigma_{z}}\right)^{-1/2} (\boldsymbol{Z_n} - \boldsymbol{Z_0})\right]= \\
	n(\boldsymbol{Z_n} - \boldsymbol{Z_0})^T \left(\boldsymbol{\hat \Sigma_{z}}^{-1/2}\right)^T \boldsymbol{\hat \Sigma_{z}}^{-1/2} (\boldsymbol{Z_n} - \boldsymbol{Z_0})
	\xrightarrow{d} \mathcal{\chi}^2(2)
	\end{split}    
	\end{equation}   
	b) follows from independence of $\boldsymbol{Z^{(1)}_n}$ and $\boldsymbol{Z^{(2)}_n}$.
\end{proof}

\end{appendices}
\end{document}